\documentclass[letterpaper]{article} 
\usepackage{aaai2026}  
\usepackage{times}  
\usepackage{helvet}  
\usepackage{courier}  
\usepackage[hyphens]{url}  
\usepackage{graphicx} 
\usepackage{natbib}  
\usepackage{caption} 
\usepackage{soul}
\usepackage{url}

\usepackage{amsthm}
\usepackage{amsfonts}
\usepackage{amsmath}
\usepackage{mathtools}
\usepackage{mathrsfs}
\usepackage{booktabs}
\usepackage{bm}
\usepackage{framed}
\usepackage[switch]{lineno}
\mathtoolsset{showonlyrefs}
\usepackage{lipsum}
\usepackage{multirow}
\usepackage{appendix}

\usepackage{algorithm}
\usepackage{algorithmic}

\usepackage{newfloat}
\usepackage{listings}
\DeclareCaptionStyle{ruled}{labelfont=normalfont,labelsep=colon,strut=off} 
\floatstyle{ruled}
\newfloat{listing}{tb}{lst}{}
\floatname{listing}{Listing}
\newtheorem{theorem}{Theorem}
\newtheorem{lemma}{Lemma}
\newtheorem{definition}{Definition}

\newcommand{\bmx}{\bm{x}}
\newcommand{\bmy}{\bm{y}}
\newcommand{\bmz}{\bm{z}}

\newcommand{\bmv}{\bm{v}}
\newcommand{\bmr}{\bm{r}}

\newcommand{\bmf}{\bm{f}}
\newcommand{\ojzj}{OneJumpZeroJump}
\newcommand{\aojzj}{OneJumpZeroJump$_{SS}$}

\newcommand{\omm}{OneMinMax}
\newcommand{\lotz}{LeadingOnesTrailingZeroes}
\nocopyright

\title{Not Just for Archiving: Provable Benefits of Reusing the Archive \\in Evolutionary Multi-objective Optimization}
\author {
    Shengjie Ren\textsuperscript{\rm 1,\rm 2},
    Zimin Liang\textsuperscript{\rm 3},
    Miqing Li\textsuperscript{\rm 3}
    Chao Qian\textsuperscript{\rm 1,\rm 2}
}
\affiliations {
    \textsuperscript{\rm 1}National Key Laboratory for Novel Software Technology, Nanjing University, Nanjing 210023, China\\
    \textsuperscript{\rm 2}School of Artificial Intelligence, Nanjing University, Nanjing 210023, China\\
    \textsuperscript{\rm 3}School of Computer Science, University of Birmingham,
Birmingham B15 2TT, U.K.\\
    rensj@lamda.nju.edu.cn, zxl525@student.bham.ac.uk, m.li.8@bham.ac.uk, qianc@nju.edu.cn
}

\usepackage{bibentry}

\begin{document}

\maketitle

\begin{abstract}

Evolutionary Algorithms (EAs) have become the most popular tool for solving widely-existed multi-objective optimization problems. In Multi-Objective EAs (MOEAs), there is increasing interest in using an archive to store non-dominated solutions generated during the search. This approach can 1) mitigate the effects of population oscillation, a common issue in many MOEAs, and 2) allow for the use of smaller, more practical population sizes. In this paper, we analytically show that the archive can even further help MOEAs through reusing its solutions during the process of new solution generation. We first prove that using a small population size alongside an archive (without incorporating archived solutions in the generation process) may fail on certain problems, as the population may remove previously discovered but promising solutions. We then prove that reusing archive solutions can overcome this limitation, resulting in at least a polynomial speedup on the expected running time. Our analysis focuses on the well-established SMS-EMOA algorithm applied to the commonly studied OneJumpZeroJump problem as well as one of its variants. We also show that reusing archive solutions can be better than using a large population size directly. Finally, we show that our theoretical findings can generally hold in practice by experiments on well-known practical optimization problems -- multi-objective 0-1 Knapsack, TSP, QAP and NK-landscape problems -- with realistic settings.

\end{abstract}


\section{Introduction}
Multi-objective optimization deals with problems in which multiple, often conflicting, objectives must be optimized simultaneously. Unlike single-objective optimization, there is generally no single optimal solution for a Multi-objective Optimization Problem (MOP). Instead, the goal is to find a set of solutions representing different optimal trade-offs among the objectives, called Pareto optimal solutions. The corresponding objective vectors of these solutions constitute the Pareto front. Thus, the goal of multi-objective optimization is to find the Pareto front or its good approximation. Evolutionary Algorithms (EAs), a class of randomized heuristic optimization algorithms inspired by natural evolution, have proven to be well-suited for solving MOPs due to their population-based nature. They have been widely applied across various real-world domains~\cite{deb2001book,qian19el}. Over the years, numerous well-established Multi-Objective EAs (MOEAs) have been developed, including the Non-dominated Sorting Genetic Algorithm II (NSGA-II)~\cite{deb-tec02-nsgaii}, the Multi-Objective Evolutionary Algorithm based on Decomposition (MOEA/D)~\cite{zhang2007moea}, and the $\mathcal{S}$-Metric Selection Evolutionary Multi-objective Optimization Algorithm (SMS-EMOA)~\cite{beume2007sms}.


\begin{table*}[t]
  \centering
  \begin{tabular}{llcc}
  \toprule
    & &\textbf{\ojzj} & \textbf{\aojzj} \\ \midrule
    \multirow{3}{*}[-5ex]{SMS-EMOA } & \multirow{3}{*}[1ex]{Archive without reuse}  & $n^{\Omega(n)}$ (Theorem~\ref{thm:ojzj1})  & $\Omega(n^k)$ (Theorem~\ref{thm:archive_benchmark})\\
    & &  [$\mu =2, k\le n/8$] & [$\mu \le n-2k+1, k\ge 13, a\ge k/4$]\\
    \cmidrule(lr){3-4}
     & \multirow{3}{*}[1ex]{Archive with reuse}  & $O(\mu n^k)$ (Theorem~\ref{thm:ojzj2})  & $O(\max\{\mu n^a/k, n^{k-a+1},n^2\log n\})$ (Theorem~\ref{thm:reusing_benchmark})\\
    & &  [$\mu \ge 2, k\ge 3$] & [$\mu \ge 2$]\\
     \cmidrule(lr){3-4}
    &Original algorithm & $O(\mu n^k)$~\cite{bian23stochastic}  & $O(\mu \cdot \max\{n^a/k, n^{k-a} \})$ (Theorem~\ref{thm:population_benchmark})\\
    &with a large population &  [$\mu \ge n-2k+3$] & [$\mu \ge n-2k+5$]\\
    \bottomrule
  \end{tabular}
  \caption{The expected number of generations of SMS-EMOA for solving \ojzj\ and \aojzj\ when 1) considering an archive without the reuse, 2) considering an archive with the reuse, and 3) employing a large population size, where $n$ denotes the problem size, $k$ and $a$ denote the parameters of \ojzj\ ($2\le k<n/2$) and \aojzj\ ($3\le k<n/2, 2\le a< k$), and $\mu$ denotes the population size.}
  \label{tab:summary}
\end{table*}

In the area of MOEAs, there has recently been growing interest in using an archive to store non-dominated solutions
generated during the search process~\cite{li2023multi}. The main reason for this practice is as follows. In MOPs, there can be infinite many Pareto optimal solutions~\cite{figueira2017easy}, and it may not be practical to use a very large population aiming for accommodating all of them. However, using a small population may easily lead to the oscillation phenomenon, even when an elite preservation mechanism is employed~\cite{knowles2003properties}. This is because, during the evolutionary process, the number of generated non-dominated solutions usually exceeds the population size. As a result, population maintenance becomes necessary to remove excess non-dominated solutions (e.g., based on criteria of crowding distance). Yet, the algorithm may later accept new solutions that are non-dominated w.r.t. the current population but are actually dominated by the previously discarded solutions. Consequently, MOEAs may end up with a final population containing a large number of sub-optimal (i.e., globally dominated) solutions. This issue affects all practical MOEAs~\cite{li2023multi}. For example, Li and Yao~\shortcite{li2019empirical} reports that on some problems like DTLZ7~\cite{deb2005scalable}, nearly half of the final population produced by NSGA-II and MOEA/D are not globally optimal (i.e., being dominated by previously discarded solutions).
This phenomenon leads to the popularity of using an archive in MOEAs that store non-dominated solutions found during the search process, e.g.,~\cite{fieldsend2003using,krause2016unbounded,brockhoff2019benchmarking,ishibuchi2020new}. Recently, theoretical studies~\cite{bian2024archive} also confirm this practice, showing that using an archive can provide a speedup, particularly when working with stochastic population updates in MOEAs~\cite{ren2025archive}.

In this paper, we theoretically show that the archive can further help MOEAs by reusing its solutions during the process of new solution generation. 
One drawback of using a small population alongside an archive is that some promising non-dominated solutions discovered by an MOEA may be removed from the population due to its limited size, preventing the algorithm from further exploiting them.
A reuse of the archive where all non-dominated solutions are preserved can ``revisit" these solutions, resulting in a speedup. Specifically, in this study we analyze the expected running time of the well-established SMS-EMOA for solving OneJumpZeroJump and its variant \aojzj, and prove that a simple way of reusing the archive (i.e., with probability 1/2, randomly choosing solutions in the archive to generate offspring) helps. 
We compare three cases of the algorithm: 1) using an archive solely to store non-dominated solutions (as is common in many MOEAs), 2) reusing the archive, and 3) employing a large population size (i.e., the original version of the algorithm).
The results are given in Table~\ref{tab:summary}. Our findings can be summarized as follows.
\begin{itemize}
    \item Reusing the archive can help find a more diverse and well-covered Pareto front. By comparing the results of Theorem~\ref{thm:ojzj1} and Theorem~\ref{thm:ojzj2} in Table~\ref{tab:summary}, we show that the expected running time of SMS-EMOA for solving \ojzj\ when only using an archive to store solutions is $n^{\Omega(n)}$, whereas it is reduced to $O(\mu n^k)$ when the archive is reused to generate new solutions. The reason is that a small population may miss some regions of the Pareto front during the early stages of the search, and re-finding those solutions becomes difficult afterward. In contrast, reusing the archived solutions enables the algorithm to re-explore the previously missed regions easily.
    \item Reusing the archive can benefit exploration. By comparing the results of Theorem~\ref{thm:archive_benchmark} and Theorem~\ref{thm:reusing_benchmark} in Table~\ref{tab:summary}, we show that the expected running time of SMS-EMOA for solving \aojzj\ using an archive that only stores solutions is $\Omega(n^k)$, whereas it is reduced to $O(\max\{\mu n^a/k, n^{k-a+1},n^2\log n\})$ (where $\mu \geq 2$ and $2 \leq a < k$) when the archive is reused. The reason is that, for most practical MOEAs like SMS-EMOA, when the number of non-dominated solutions exceeds the population size, diversity maintenance mechanisms in the objective space are typically employed to prune the population. As a result, some non-dominated solutions that are sparse in the decision space (often beneficial for exploration), but are crowded in the objective space, are likely to be discarded. Since such solutions are preserved in the archive, allowing reusing the archive helps the algorithm explore sparse regions in the decision space.
    \item Reusing the archive with a small population can be better than directly using a large population size. By comparing the results in the last two rows of Table~\ref{tab:summary}, we show that for SMS-EMOA solving \ojzj\ and \aojzj, reusing the archive can reduce the upper bound on the expected running time compared to using a large population size. The reason is that reusing the archive allows for a constant population size $\mu \ge 2$, which brings a speedup of $\Theta(n)$. 
    On \aojzj, this improvement becomes evident only when $a > k/2$. 
    Actually, Bian et al.~\shortcite{bian2024archive} showed that compared to using a large population size, using an archive can provide a polynomial-speedup for SMS-EMOA solving OneMinMax and LeadingOnesTrailingZeroes. In this paper, we also show that reusing the archive does not diminish this advantage.
\end{itemize}

We also conduct experiments on the artificial problems studied in this paper and four well-known practical problems: the multi-objective 0/1 knapsack~\cite{teghem_multi-objective_1994}, TSP~\cite{ribeiro2002study}, QAP~\cite{knowles2003instance}, and NK-landscapes~\cite{Aguirre2004}. The results validate our theoretical findings.

It is worth pointing out that there have been many important studies on the running time of MOEAs over the past two decades. Early theoretical works~\cite{laumanns-nc04-knapsack,LaumannsTEC04,Neumann07,Giel10,Neumann10,doerr2013lower,Qian13,qian-ppsn16-hyper,bian2018tools} mainly focus on simple MOEAs like \mbox{(G)SEMO}. Recently, researchers have begun to examine practical MOEAs.
Huang \textit{et al.}~\shortcite{huang2021runtime} investigated MOEA/D. Zheng \textit{et al.}~\shortcite{zheng2021first} conducted the first theoretical analysis of NSGA-II. Bian \textit{et al.}~\shortcite{bian23stochastic} analyzed the running time of SMS-EMOA and showed that stochastic population update can bring acceleration. Wietheger and Doerr~\shortcite{wietheger23nsgaiii} showed that NSGA-III~\cite{deb2014nsgaiii} outperforms NSGA-II on $3$OneMinMax. Ren \textit{et al.}~\shortcite{Ren2024spea2} analyzed SPEA2.
Some other works on well-established MOEAs include~\cite{bian2022better,zheng2023manyobj,cerf2023first,dang2023analysing,dang2023crossover,doerr2023lower,doerr2023crossover,dang2024level,Opris2024nsgaiii,ren2024multimodel,wietheger2024near,doerr2025runtime,doerr2025speeding,opris2025first}. There are also a series of theoretical studies showing that MOEAs can achieve good approximation guarantees for single-objective submodular optimization with constraints~\cite{qian2015subset,Friedrich2015sub,qian2017subset,qian2019sub}


\section{Preliminaries}\label{sec-preliminary}

In this section, we first give basic concepts in multi-objective optimization, followed by a description of the SMS-EMOA algorithm and the proposed archive reusing mechanism. Lastly, we describe the \ojzj\ problem and its variant \aojzj, studied in this paper.

\subsection{Multi-objective Optimization}

Multi-objective optimization aims to optimize two or more objective functions simultaneously, as shown in Definition~\ref{def_MO}. The objectives are typically conflicting, meaning that there is no canonical complete order in the solution space $\mathcal{X}$. To compare solutions, we use the \emph{domination} relationship in Definition~\ref{def_Domination}.  
A solution is \emph{Pareto optimal} if no other solution in $\mathcal{X}$ dominates it. The set of objective vectors corresponding to all Pareto optimal solutions forms the \emph{Pareto front}. The goal of multi-objective optimization is to find the Pareto front or its good approximation.

\begin{definition}[Multi-objective Optimization]\label{def_MO}
	Given a feasible solution space $\mathcal{X}$ and objective functions $f_1,f_2,\ldots, f_m$, multi-objective optimization can be formulated as
	\[
	\max_{\bmx\in
		\mathcal{X}}\bmf(\bmx)=\max_{\bmx \in
		\mathcal{X}} \big(f_1(\bmx),f_2(\bmx),...,f_m(\bmx)\big).
	\]
\end{definition}
\begin{definition}\label{def_Domination}
	Let $\bm f = (f_1,f_2,\ldots, f_m):\mathcal{X} \rightarrow \mathbb{R}^m$ be the objective vector. For two solutions $\bmx$ and $\bmy\in \mathcal{X}$:
	\begin{itemize}
		\item $\bmx$ \emph{weakly dominates} $\bmy$ (denoted as $\bmx \succeq \bmy$) if for all $1 \leq i \leq m, f_i(\bmx) \geq f_i(\bmy)$;
		\item $\bmx$ \emph{dominates} $\bmy$ (denoted as $\bmx\succ \bmy$) if $\bm{x} \succeq \bmy$ and $f_i(\bmx) > f_i(\bmy)$ for some $i$;
		\item $\bmx$ and $\bmy$ are \emph{incomparable} if neither $\bmx\succeq \bmy$ nor $\bmy\succeq \bmx$.
	\end{itemize}
\end{definition}

\subsection{SMS-EMOA and Archiving Mechianism}

SMS-EMOA~\cite{beume2007sms} presented in Algorithm~\ref{alg:sms} is a popular MOEA, which employs non-dominated sorting and hypervolume indicator to update the population. It starts with an initial population of $\mu$ solutions (line~1). In each generation, it randomly selects a solution $\bm{x}$ from the current population (line~3) for reproduction. Afterwards, bit-wise mutation flips each bit of $\bmx$ with probability $1/n$ to produce an offspring $\bmx'$ (line~4). Then, the union of the current population and the offspring solution is partitioned into non-dominated sets $R_1,\ldots,R_v$ (line~5), where $R_1$ contains all non-dominated solutions in $P\cup \{\bm{x}'\}$, and $R_i$ ($i \ge 2$) contains all non-dominated solutions in $P\cup \{\bm{x}'\} \setminus \cup_{j=1}^{i-1} R_j$. A solution $\bm{z} \in R_v$ is then removed (lines~6--7) by minimizing $\Delta_{\bmr}(\bm{x}, R_v) := HV_{\bmr}(R_v) - HV_{\bmr}(R_v \setminus \{\bm{x}\})$, where $HV_{\bmr}(X)$ denotes the hypervolume of the solution set $X$ with respect to a reference point $\bmr \in \mathbb{R}^m$ ($\forall i, r_i\le \min_{\bmx\in \mathcal{X}}f_i(\bmx)$), i.e., the volume between the reference point and the objective vectors of the solution set. A larger hypervolume indicates better approximation of the Pareto front in terms of convergence and diversity. For bi-objective problems, as defined in the original SMS-EMOA~\cite{beume2007sms}, the algorithm omits the reference point and directly preserves the two boundary points, i.e., the two solutions with the maximum $ f_1 $ and $ f_2 $ values, respectively, allowing the hypervolume contribution to be calculated accordingly.

\begin{algorithm}[tb]
	\caption{SMS-EMOA}
	\label{alg:sms}
	\textbf{Input}: objective functions $f_1,f_2\ldots,f_m$, population size $\mu$\\
    \textbf{Output}: $\mu$ solutions from $\{0,1\}^n$
	\begin{algorithmic}[1] 
		\STATE $P\leftarrow \mu$ solutions uniformly and randomly selected from $\{0,\! 1\}^{\!n}$ with replacement;
		\WHILE{criterion is not met}
		\STATE select a solution $\bmx$ from $P$ uniformly at random;
		\STATE apply bit-wise mutation on $\bmx$ to generate $\bmx'$;
		\STATE partition $P\!\cup \!\{\bmx'\!\}$ into non-dominated sets $R_1\!,...,\!R_v$;
		\STATE let $\bmz=\arg\min_{\bmx\in R_v}\Delta_{\bmr}(\bmx,R_v)$;
		\STATE $P\leftarrow (P\cup \{\bmx'\})\setminus \{\bmz\}$
		\ENDWHILE
		\RETURN $P$
	\end{algorithmic}
\end{algorithm}

An archive in MOEAs is used to store non-dominated solutions found during the search, preventing them from being discarded due to limited population capacity. Once a new solution is generated, the solution will be tested if it can enter the archive. If there is no solution in the archive that weakly dominates the new solution, then the solution will be placed in the archive, and meanwhile those solutions dominated by the new solution will be deleted from the archive. Algorithmic steps incurred by adding an archive in SMS-EMOA are given as follows. In Algorithm~\ref{alg:sms}, an empty set $A$ is initialized in line~1, the set $A$ instead of $P$ is returned in the last line, and
the following lines are added after line~4:
\begin{framed}\vspace{-0.8em}
  \begin{algorithmic}
  \IF{$\not \exists \bmz \in A$ such that $\bmz \succeq \bmx'$}
  \STATE $A \leftarrow (A \setminus\{\bmz \in A \mid \bmx' \succ \bmz\}) \cup \{\bmx'\}$
  \ENDIF
\end{algorithmic}\vspace{-0.8em}
\end{framed}\noindent

\subsection{Reusing the Archive}

Reusing the archive means allowing the solutions stored in it to re-participate in the evolutionary process. It can be materialized in different ways. Here we consider a straightforward one. During the parent selection, SMS-EMOA uniformly randomly selects a parent solution from population $ P $ with probability $ 1/2 $, and otherwise selects a parent solution from archive $ A $. Specifically, the following lines replace line~3 in Algorithm~\ref{alg:sms}:

\begin{framed}\vspace{-0.8em}
  \begin{algorithmic}
  \STATE sample $u$ from uniform distribution over $[0,1]$;
  \IF{$u<1/2$}
  \STATE select a solution $\bm{x}$ from $P$ uniformly at random;
  \ELSE
  \STATE select a solution $\bm{x}$ from $A$ uniformly at random;
  \ENDIF
\end{algorithmic}\vspace{-0.8em}
\end{framed}\noindent
In this paper, we set the archive reusing rate to $ 1/2 $, which can be adjusted in practice.

\subsection{Benchmark Problems}

Next, we introduce benchmark problems \ojzj\ and its variant \aojzj, studied in this paper. The \ojzj\ problem as presented in Definition~\ref{def:ojzj} is constructed based on the Jump problem~\cite{doerr2020theory}, and has been widely used in MOEAs' theoretical analyses~\cite{doerr2021ojzj,doerr2023nsgaojzj,lu2024imoea,ren2024multimodel}. Its first objective is the same as the Jump problem, which aims to maximize the number of 1-bits of a solution, except for a valley around $1^n$ (the solution with all 1-bits) where the number of 1-bits should be minimized. The second objective is isomorphic to the first, with the roles of 1-bits and 0-bits reversed. The Pareto front of the \ojzj \ problem is
$ \{(c, n+2k-c)\mid c\in[2k..n]\cup\{k, n+k\}\}$, whose size is $n-2k+3$, and the Pareto optimal solution corresponding to $(c, n+2k-c)$ is any solution with $(c-k)$ 1-bits. We use $S_I^* = \{ \bmx \mid |\bmx|_1 \in [k..n-k] \}$ and $F_{I}^*=\{(c, n+2k-c)\mid c\in[2k..n]\}$ to denote the internal part of Pareto set and Pareto front, respectively.

\begin{definition}[Doerr and Zheng~\shortcite{doerr2021ojzj}]\label{def:ojzj}
	The \ojzj\ problem is to find $n$ bits binary strings which maximize
	\[ f_1(\bmx) = \begin{cases}
		k+|\bmx|_1, & \text{if }|\bmx|_1 \leq n-k\text{ or } \bmx=1^n,\\
		n-|\bmx|_1, & \text{else},
	\end{cases}\]
	\[f_2(\bmx) = \begin{cases}
		k+|\bmx|_0, & \text{if }|\bmx|_0 \leq n-k\text{ or } \bmx=0^n,\\
		n-|\bmx|_0, & \text{else},
	\end{cases}\]
	where $k\in \mathbb{Z} \wedge 2\le k<n/2$, and $|\bmx|_1$ and $|\bmx|_0$ denote the number of 1-bits and 0-bits in $\bmx \in \{0,1\}^n$, respectively.
\end{definition}

\begin{figure}\centering
        \begin{minipage}[c]{0.48\linewidth}\centering
		\includegraphics[width=1\linewidth]{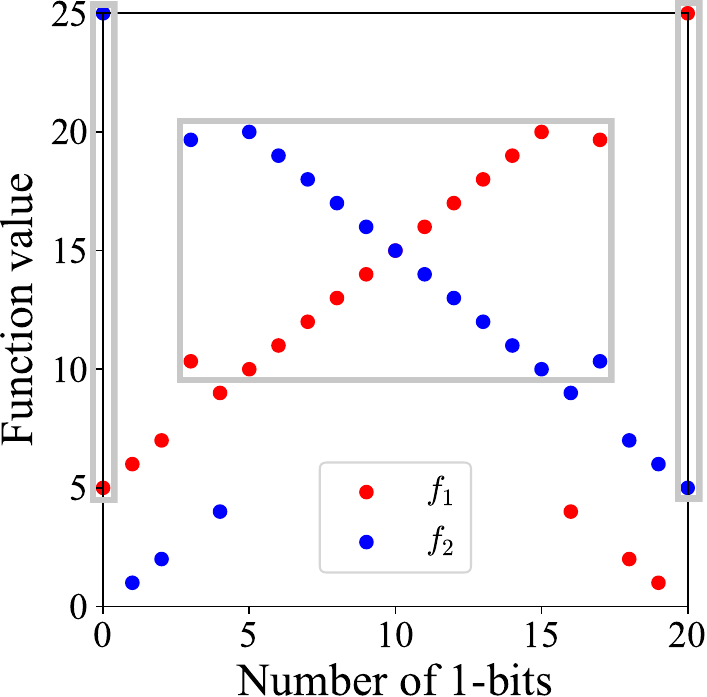}
	\end{minipage}
        \hspace{0.1em}
        \begin{minipage}[c]{0.48\linewidth}\centering
		\includegraphics[width=1\linewidth]{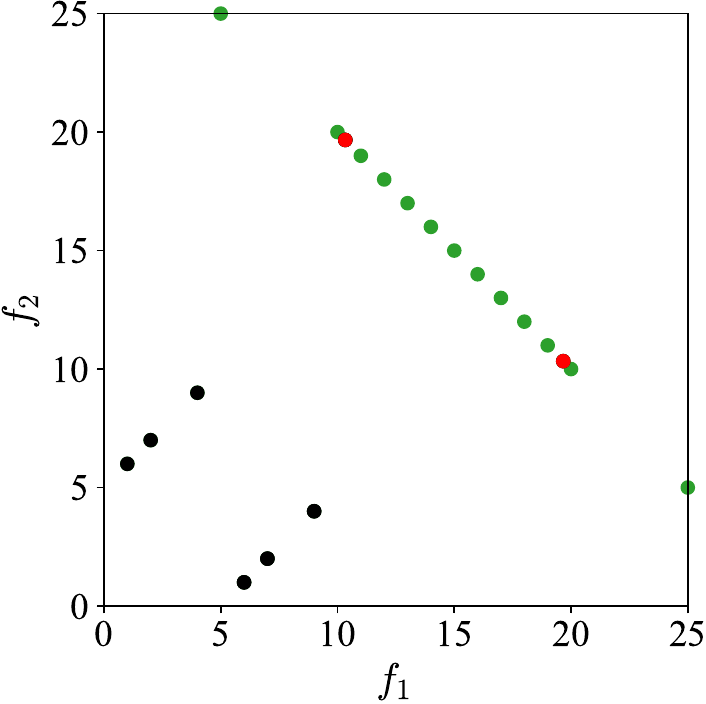}
	\end{minipage}
	\caption{Illustration of the \aojzj\  problem when $n=20$, $k=5$ and $a=2$.  The left subfigure: the function values $f_1$ and $f_2$ vs. the number of 1-bits of a solution; the right subfigure: $f_2$ vs. $f_1$.}\label{fig:ojzj}
\end{figure}

In this paper, we propose a new variant of \ojzj, denoted as \aojzj\ (\ojzj\ with stepping stone), as presented in Definition~\ref{new_ojzj}. Compared to \ojzj, the difference lies in that, among solutions whose number of 1-bits falls in the intervals $[1,\ldots,k-1]$ and $[n-k+1,\ldots,n-1]$, the solutions with $k-a$ and $n-(k-a)$ 1-bits are non-dominated. They lie between the internal Pareto front and the extremal points (i.e., $0^n$ and $1^n$) in the solution space, but their corresponding objective vectors are in a crowded region of the objective space. The right subfigure of Figure~\ref{fig:ojzj} shows an example of the objective vectors and the Pareto front. The red points indicate the two newly introduced Pareto front points, which we denote as $G = \left\{ \left(n - 1/n, 2k + 1/n \right), \left(2k + 1/n, n - 1/n \right) \right\}$. We also use $S_I^* = \{ \bmx \mid |\bmx|_1 \in [k..n-k] \}$ and $F_{I}^*=\{(c, n+2k-c)\mid c\in[2k..n]\}$ to denote the internal part of Pareto set and Pareto front. The left subfigure of Figure~\ref{fig:ojzj} shows the values of $f_1$ and $f_2$ with respect to the number of $1$-bits of a solution. The boxed regions represent the Pareto front. We use \aojzj\ to characterize a class of problems that some non-dominated solutions are sparse in the solution space, making them helpful for exploration, but crowded in the objective space, reflecting scenarios where the solution space and objective space exhibit inconsistency.

\begin{definition} \label{new_ojzj}
The \aojzj\ problem is to find $n$ bits binary strings which maximize
\begin{align}
    &f_1(\bmx) = \begin{cases}
        2k+1/n, & \text{if }|\bmx|_1 = k-a,\\
        n-1/n, & \text{else if }|\bmx|_1 = n-(k-a),\\
		k+|\bmx|_1, & \text{else if }|\bmx|_1 \le n-k  \text{ or } \bmx=1^n,\\
		n-|\bmx|_1, & \text{else},
	\end{cases} \\
&f_2(\bmx) = f_1(\bar{\bmx}),
\end{align}
	where $a,k\in \mathbb{Z},  3\le k<n/2, 2\le a < k$, and $\bar{\bmx} = (1-x_1,\cdots,1-x_n)$.
\end{definition}

\section{Provable Benefits of Reusing the Archive}

In this section, we prove that reusing the archive is beneficial for SMS-EMOA solving \ojzj\ and \aojzj, compared to using an archive without reuse. 
Note that SMS-EMOA generates and evaluates only one offspring solution in each generation, thus its running time is equivalent to the number of generations. 
We give Stirling's formula in Lemma~\ref{lemma:sti} that will be frequently used in the proofs.
\begin{lemma}[Robbins~\shortcite{robbins1955remark}]\label{lemma:sti}
    For a positive integer $n$, we have Stirling's formula as $\sqrt{2\pi n}(n/e)^n \le n! \le e^{1/12}\sqrt{2\pi n}(n/e)^n$
\end{lemma}

\subsection{Archive Reuse Enhances Pareto Front Coverage}\label{sec:coverage}

First, we theoretically show that reusing the archive could enhance Pareto front coverage. In Theorem~\ref{thm:ojzj1}, we prove that when using a small population size $\mu=2$ and an archive without reuse, the expected number of generations for SMS-EMOA solving \ojzj\ is $n^{\Omega(n)}$. We use this example to illustrate a scenario where, during the early stages of the search, a small population may miss some regions easy to reach, and re-finding these regions later can become difficult as the population becomes more uniformly distributed and stabilizes. The landscape of OneJumpZeroJump (consisting of a plain and two valleys) is simpler than those of real-world problems, and our analysis relies on a population size $\mu=2$. For more complex, real-world problems, similar issues may arise with larger populations. 

\begin{theorem}\label{thm:ojzj1}
  For SMS-EMOA solving OneJumpZeroJump with $k \le n/8$, if using a population size $\mu=2$ and an archive without reuse, then the expected running time for finding the Pareto front is $n^{\Omega(n)}$.
\end{theorem}

\begin{proof}
We first show that $A_1$, which denotes the event that the two initial solutions belong to $S_I^*$ (i.e., solutions with number of 1-bits between $k$ and $n - k$) and are located on opposite sides of the central region of the Pareto front (i.e., solutions with $\lfloor n/2\rfloor$ 1-bits), occurs with probability at least $1 - e^{-\Omega(n)} - 1/\Omega(\sqrt{n})$.
Then conditioned on the occurrence of event $A_1$, we prove that $A_2$, which denotes the event that the central point on the Pareto front (i.e., solutions with $\lfloor n/2\rfloor$ 1-bits) can only be generated from $0^n$ and $1^n$, occurs with probability at least $e^{-8e(e-1)-1/\Omega(n)}$. Finally, conditioned on the occurrence of events $A_1$ and $A_2$, we show that the expected number of generations for finding the whole Pareto front is at least $n^{\Omega(n)}$. 

Let $A_1$ denote the event that the two initial solutions belong to $S_I^*$ (i.e., solutions with number of 1-bits between $k$ and $n - k$) and are located on opposite sides of the central region of the Pareto front. Specifically, for even $n$, the two initial solutions $\bm{x}$ and $\bm{y}$ satisfy $ k \le |\bm{x}|_1 < n/2$ and $ n/2 < |\bm{y}|_1 \le n-k$; for odd $n$, the two initial solutions $\bm{x}$ and $\bm{y}$ satisfy $k \le |\bm{x}|_1 < (n - 1)/2$ and $(n -1)/2 < |\bm{y}|_1 \le n-k$. Then, we show that event $A_1$ occurs with probability at least $1/2- e^{-\Omega(n)} -1/\Omega(\sqrt{n})$. For an initial solution $\bm{x}$, it is generated uniformly at random, i.e., each bit in $\bm{x}$ can be $1$ or $0$ with probability $1/2$, respectively. Then, by Chernoff bound and $n-2k=\Omega(n)$, $\bmx$ has number of 1-bits between $k$ and $n-k$ with probability at least $1-e^{\Omega(n)}$. For even $n$, the number of 1-bits in $\bm{x}$ is exactly $n/2$ with probability of 
\begin{align}\label{eq:binom_stiling1}
    \binom{n}{n/2} \cdot \frac{1}{2^n} &= \frac{n!}{(n/2)!(n/2)!}\cdot \frac{1}{2^n}\\
    &\le \frac{e^{1/12} \sqrt{2\pi n}(n/e)^n}{\pi n (n/(2e))^n}\cdot \frac{1}{2^n}\\
    &= \sqrt{\frac{2e^{1/6}}{\pi n}},
\end{align}
where the inequality holds by Stirling's formula in Lemma~\ref{lemma:sti}. Thus, by symmetry, we have $\Pr[k \le |\bm{x}|_1 < n/2] = \Pr[n/2 < |\bm{x}|_1 \le n-k]\ge (1-e^{-\Omega(n)} - \sqrt{2e^{1/6}/(\pi n)})/2$. Then, we can derive that 
\begin{align}\label{Eq:lower_A0}
    \Pr[A_1] =& \Pr[(k \le |\bm{x}|_1 < n/2) \wedge (n/2 < |\bm{y}|_1 \le n-k)] \\&+ \Pr[(k \le |\bm{y}|_1 < n/2) \wedge (n/2 < |\bm{x}|_1 \le n-k)]\\
    \ge&  2 \cdot \Big(\Big(1- e^{-\Omega(n)} -\sqrt{\frac{2e^{1/6}}{\pi n}}\Big)/2\Big)^2\\
    \ge& \frac{1}{2} - e^{\Omega(n)}-\frac{1}{\Omega(\sqrt{n})}.
\end{align}
For odd $n$, the number of 1-bits in $\bm{x}$ is exactly $(n-1)/2$ with probability of
\begin{align}\label{eq:binom_stiling2}
    &\binom{n}{(n-1)/2} \cdot \frac{1}{2^n} = \frac{n!}{((n-1)/2)!((n+1)/2)!}\cdot \frac{1}{2^n}\\
    &\le \frac{(n-1)!}{((n-1)/2)!((n-1)/2)!}\cdot 2 \cdot \frac{1}{2^n}\\
    &\le \frac{e^{1/12} \sqrt{2\pi (n-1)}((n-1)/e)^{(n-1)}}{\pi (n-1) ((n-1)/(2e))^{(n-1)}}\cdot \frac{1}{2^{n-1}}\\
    &= \sqrt{\frac{2e^{1/6}}{\pi (n-1)}},
\end{align}
where the second inequality also holds by Stirling's formula. Then, by symmetry, we have $\Pr[|\bmx|_1 = (n+1)/2] \le \sqrt{2e^{1/6}/(\pi(n-1))}$. Thus, similar to the analysis in Eq.~\ref{Eq:lower_A0}, we can derive that
\begin{align}
    \Pr&[A_1] \\={}& \Pr[(k \le |\bm{x}|_1 < (n-1)/2) \wedge ((n-1)/2 < |\bm{y}|_1 \le n-k)] \\+& \Pr[(k \le |\bm{y}|_1 < (n-1)/2) \wedge ((n-1)/2 < |\bm{x}|_1 \le n-k)]\\
\ge{}& \Pr[(k \le |\bm{x}|_1 < (n-1)/2) \wedge ((n+1)/2 < |\bm{y}|_1 \le n-k)] \\{}+& \Pr[(k \le |\bm{y}|_1 < (n-1)/2) \wedge ((n+1)/2 < |\bm{x}|_1 \le n-k)]\\
    \ge{}&  2 \cdot \Big(\Big(1-e^{-\Omega(n)}-2\sqrt{\frac{2e^{1/6}}{\pi (n-1)}}\Big)/2\Big)^2\\
    ={}& \frac{1}{2}- e^{-\Omega(n)} -\frac{1}{\Omega(\sqrt{n})}.
\end{align}
Thus, event $A_1$ occurs with probability at least $1/2 - e^{-\Omega(n)}- 1/\Omega(\sqrt{n})$. 

Conditioned on the occurrence of $A_1$, let $A_2$ denote the event that the central point on the Pareto front (i.e., solutions with $n/2$ 1-bits for even $n$ and $(n-1)/2$ 1-bits for odd $n$) can only be generated from $0^n$ or $1^n$. Then, we show that $A_2$ occurs with probability at least $e^{-8e(e - 1)} \cdot e^{-1/\Omega(n)}$. We note that in the original SMS-EMOA~\cite{beume2007sms}, when solving bi-objective problems, the algorithm directly preserves the two boundary points, i.e., the two solutions with the maximum $ f_1 $ and $ f_2 $ values respectively. Thus, the maximum $f_1$ and $f_2$ value among the Pareto optimal solutions in $P\cup \{\bm{x}'\}$ will not decrease, where $\bm{x}'$ is the offspring solution generated in each generation. 
 Assume that $n$ is even and the solution with the maximum $f_2$ value in the population is $\bm{x}^*$. According to the definition of \ojzj\ in Def~\ref{def:ojzj}, $\bm{x}^*$ has the minimum number of 1-bits. Suppose that in some generation, $\bmx^*$ such that $|\bmx^*|_1 = i$, is selected as the parent solution. By flipping a single 1-bits with probability $(i/n)\cdot (1-1/n)^{n-1} \ge i/(en)$, the maximum $f_2$ value of population increases. Or by flipping $n/2-i$ 0-bits with probability $(1-1/n)^{n/2-i}\cdot \binom{n-i}{n/2-i}/ n^{n/2-i} \le  \binom{n-i}{n/2-i}/( n^{n/2-i})$, a solution with $n/2$ 1-bits is found. Thus, if choosing $\bmx^*$ as the parent, the conditional probability of increasing $f_2$ value rather than finding a solution with $n/2$ 1-bits is at least
 \begin{align}\label{eq:decrease_by_1}
     &\frac{i/(en)}{i/(en) + \binom{n-i}{n/2-i}/(n^{n/2-i})} = \frac{1}{1 + (en/i) \binom{n-i}{n/2-i}/n^{n/2-i}}\\
     &= \frac{1}{1+(en/i)\cdot \frac{(n-i)!}{(n/2-i)!(n/2)! n^{n/2-i}} } \ge \frac{1}{1+ \frac{en}{i(n/2-i)!}},
 \end{align}
where the inequality holds by $(n-i)!/(n/2)! < n^{n/2-i}$. Let $a_i := en/((n/2-i)!i)$. Then, when choosing $\bmx^*$ as the parent solution, $f_2$ decreases to $n$ (i.e., the solutions with $k$ 1-bits is found) before finding the solution with $n/2$ 1-bits with probability at least
\begin{align}
    \prod_{i=k+1}^{n/2-1} \frac{1}{1+a_i} \ge e^{-\sum_{i=k+1}^{n/2-1} a_i},
\end{align}
where the inequality holds by $1+b\le e^b$ for any $b\in \mathbb{R}$. For $\sum_{i=k+1}^{n/2-1}a_i$, we have
\begin{align}
    \sum_{i=k+1}^{n/2-1} a_i &=  \sum_{i=\lceil n/4 \rceil}^{n/2-1} a_i + \sum_{i= k+1 }^{\left\lceil n /4 \right\rceil -1} a_i\\
    &= \sum_{i=\lceil n/4 \rceil}^{n/2-1} \frac{en}{i(n/2-i)!} + \sum_{i= k+1 }^{\left\lceil n /4 \right\rceil -1} \frac{en}{i(n /2 -i)!}\\
    &\le \sum_{i=\lceil n/4 \rceil}^{n/2-1} 4e\cdot \frac{1}{(n/2-i)!} + \sum_{j = \left\lfloor n /4 \right\rfloor + 1 }^{n /2 - k -1} \frac{en}{j!}\\
    &\le 4e\sum_{j=1}^{\lfloor n/4 \rfloor} \frac{1}{j!} + 4e \sum_{j = \left\lfloor n/ 4 \right\rfloor + 1 }^{n /2 - k -1}  \frac{1}{(j-1)!}\\
    &= 4e(\sum_{j=1}^{n /2 - k -2} \frac{1}{j!} + \frac{1}{\lfloor n/4 \rfloor}!)\\
    &\le 4e(e-1+\frac{1}{\lfloor n/4 \rfloor !}) \approx 4e(e-1),
\end{align}
where the last inequality holds by $\sum_{j=1}^{\infty} 1/j! = e-1$. Thus, when choosing $\bmx^*$ as the parent, a solution with $k$ 1-bits is found before the solutions with $n/2$ 1-bits with probability at least $e^{-4e(e-1)}$. Then, similar to the analysis in Eq.~\eqref{eq:decrease_by_1}, when choosing $\bmx^*$ with $k$ 1-bits as the parent solution, the probability of finding $0^n$ before the solution with $n/2$ 1-bits is at least
\begin{align}
    &\frac{1/(en^k)}{1/(en^k) + \binom{n-k}{n/2-k}/(n^{n/2-k})} = \frac{1}{1 + en^k \binom{n-k}{n/2-k}/n^{n/2-k}} \\
    &= \frac{1}{1+en^k\cdot \frac{(n-k)!}{(n/2-k)!(n/2)! n^{n/2-k}}}\\
    &\ge \frac{1}{1+ \frac{en^k}{(n/2-k)!}}\ge  \exp\Big(-\frac{en^k}{(n/2-k)!}\Big) \\
    &\ge \exp\Big(-\frac{en^k}{\sqrt{2\pi(n /2-k)}((n /2 -k)/e )^{n /2-k}}\Big)\\
    &\ge e^{-1/\Omega(n)},
\end{align}
where the first inequality holds by $(n-k)!/(n/2)! < n^{n/2-k}$, the second inequality holds by $1+b\le e^b$ for any $b\in \mathbb{R}$, the third inequality holds by Stirling's formula $\sqrt{2\pi n} (n/e)^{n} \le n!$, and the last inequality holds by $k\le n/8$. In summary, when the algorithm selects the solution $\bm{x}^*$ with the maximum $f_2$ value as the parent, the probability that it finds $0^n$ before finding a solution with $n/2$ 1-bits is at least  $e^{-4e(e-1)} \cdot e^{-1/\Omega(n)}$. Symmetrically, when the solution $\bm{y}^*$ with the maximum $f_1$ value is selected as the parent, the algorithm reaches $1^n$ before finding a solution with $n/2$ 1-bits with the same probability. Thus, the event $A_2$ (i.e., a solution with $n/2$ 1-bits is first generated from $0^n$ or $1^n$) happens with probability at least $(e^{-4e(e-1)} \cdot e^{-1/\Omega(n)})^2 = e^{-8e(e-1)-1/\Omega(n)}$. 
For odd $n$, the analysis is similar to the case above, by replacing the central point on the Pareto front from solutions with $n/2$ 1-bits to solutions with $(n - 1)/2$ 1-bits.

Finally, conditioned on the occurrence of events $A_1$ and $A_2$, we show that the expected number of generations for finding the whole Pareto front is at least $n^{\Omega(n)}$. In each generation, the SMS-EMOA can find the solution with $n/2$ 1-bits (for even $n$) by selecting $0^n$ (or $1^n$) from the population, flipping $n/2$ 1-bits (or 0-bits) with probability
\begin{align}
    &(1-1/n)^{n/2} \cdot \frac{\binom{n}{n/2}}{n^{n/2}} \le \frac{1}{\sqrt{e}} \cdot \frac{e^{1/12} \sqrt{2\pi n}(n/e)^n}{\pi n (n/(2e))^n} \cdot \frac{1}{n^{n/2}} \\
    &\le  \sqrt{\frac{2}{e^{5/6}\pi n}} \cdot \Big(\frac{4}{n}\Big)^{n/2} = 1/n^{\Omega(n)},
\end{align}
where the first inequality holds by Stirling's formula. For odd $n$, the analysis is similar. Combining the analysis, the expected number of generations for finding the whole Pareto front is at least $(1/2- e^{-\Omega(n)} -1/\Omega(\sqrt{n}))\cdot e^{-8e(e-1)-1/\Omega(n)} \cdot n^{\Omega(n)} = n^{\Omega(n)}$. Thus, the theorem holds.
\end{proof}

Through analysis of Theorem~\ref{thm:ojzj1}, we find that using a small population may miss the central region of the Pareto front during the early stages of search, and re-finding these solutions becomes increasingly difficult as the search progresses. 
In Theorem~\ref{thm:ojzj2}, we prove that when reusing the archive, the expected running time
for SMS-EMOA solving \ojzj\ is $O(\mu n^k)$.

\begin{theorem}\label{thm:ojzj2}
  For SMS-EMOA solving OneJumpZeroJump with $k\ge 3$ and $n-2k = \Omega(n)$, if using a population size $\mu\ge 2$ and reusing the archive, then the expected running time for finding the Pareto front is $O(\mu n^k)$.
\end{theorem}

\begin{proof}
We divide the running process into three phases. The first phase starts after initialization and finishes until one point on $F_I^*$ (i.e., objective vector $\bmf(\bmx)$ such that $|\bmx|_1\in [k..n-k]$) is found; the second phase starts after the first phase and finishes until $1^n$ and $0^n$ are found; the third phase starts after the second phase and finishes until the whole internal Pareto front $F_I^*$ is found.

For the first phase, we prove that the expected number of generations for finding one point on $F_I^*$ is $O(1)$. During population initialization, by the Chernoff bound and $n-2k=\Omega(n)$, all solutions in the initial population has at most $k-1$ 1-bits or at least $n-k+1$ 1-bits with probability at most $\exp(-\Omega(n))^\mu$. Without loss of generality, we assume that one solution $\bmy$ has at most $k-1$ 1-bits. Then similar to the analysis in~\cite{bian23stochastic}, selecting $\bmy$ and flipping $k-|\bmy|_1$ 0-bits can generate a point on $F_I^*$, whose probability is at least
\begin{align}
    &\frac{1}{\mu}\cdot \frac{\binom{n-|\bmy|_1}{k-|\bmy|_1}}{n^{k-|\bmy|_1}} \cdot \Big(1-\frac{1}{n}\Big)^{n-k+|\bmy|_1} \ge \frac{1}{e\mu}\cdot \frac{\binom{n-|\bmy|_1}{k-|\bmy|_1}}{n^{k-|\bmy|_1}}\\
    &= \frac{1}{e\mu} \cdot \frac{(n-|\bmy|_1)!}{(k-|\bmy|_1)!(n-k)!n^{k-|\bmy|_1}} \ge \frac{1}{e\mu} \cdot \frac{n!}{k!(n-k!)n^k}\\
    &= \frac{\binom{n}{k}}{e\mu n^k} \ge \Big(\frac{n}{k}\Big)^k  \cdot \frac{1}{e\mu n^k} = \frac{1}{e\mu k^k},
\end{align}
where the second inequality holds by $n!/(n-|\bmy_1|)! \le n^{|\bmy|_1}$. Thus, the expected number of generations for finding a point on $F_I^*$ is at most $\exp(-\Omega(n))^\mu \cdot e\mu k^k = O(1)$.

For the second phase, we show that the expected number of generations for finding $1^n$ is $O(\mu n^k)$, and by symmetry, the same bound holds for finding $0^n$. We note that in the original SMS-EMOA~\cite{beume2007sms}, when solving bi-objective problems, the algorithm omits the reference point and directly preserves the two boundary points, i.e., the two solutions with the maximum $ f_1 $ and $ f_2 $ values respectively (with ties broken randomly in the case of identical objective vectors), and computes the hypervolume contribution of the remaining solutions based on these two points. Thus, the maximum $f_1$ value among the Pareto optimal solutions in $P\cup \{\bm{x}'\}$ will not decrease, where $\bm{x}'$ is the offspring solution generated in each generation. We assume that $1^n$ has not been found; otherwise, the second phase finishes. Now, we consider the increase of the maximum $f_1$ value among the Pareto optimal solutions found. In each generation, SMS-EMOA chooses the population $P$ as the parent population with probability $1/2$, the solution $\bm{x}$ with the maximum $f_1$ value is selected with probability $1/\mu$, and generates a solution with more 1-bits that only one of the 0-bits is flipped by bit-wise mutation with probability $((n-|\bm{x}|_1)/n)\cdot (1-1/n)^{n-1}$. This implies that the probability of generating a solution with more than $|\bm{x}|_1$ 1-bits in one generation is at least
\begin{equation}\label{eq:ojzj1}
 \frac{1}{2}\cdot\frac{1}{\mu}\cdot \frac{n-|\bm{x}|_1}{n}\cdot \Big(1-\frac{1}{n}\Big)^{n-1}\ge \frac{n-|\bm{x}|_1}{2e\mu n}.
\end{equation}
Thus, the expected number of generations for increasing the maximum $f_1$ value to $n$, i.e., finding a solution with $n-k$ 1-bits, is at most $\sum_{i=k}^{n-k-1} 2e\mu n/(n-i)=O(\mu n \log n)$. Then, a solution $\bm{x}$ with $n - k$ 1-bits remains preserved in the population $P$ until $1^n$ is discovered, due to having the highest $f_1$ value. In each generation, SMS-EMOA selects population $P$ as the parent population with probability $1/2$, selects the solution $\bm{x}$ from $P$ with probability $1/\mu$, and generates $1^n$ by flipping all remaining 0-bits with probability $(1/n^k)\cdot (1-1/n)^{n-k}$. The probability of these events happening in a single generation is at least $(1/2)\cdot (1/\mu) \cdot(1/n^k)\cdot (1-1/n)^{n-k} \ge 1/(2e\mu n^k)$. Thus, the expected number of generations for finding $1^n$ after the first phase is at most $O(\mu n\log n+  2e\mu n^k) = O(\mu n^k)$.

For the third phase, we show that the expected number of generations for finding the whole internal Pareto front $F_I^*$ is $O(n^2\log n)$. After the first phase, at least one point on $F_I^*$ (i.e., a solution with $k$ 1-bits and a solution with $n-k$ 1-bits) is found. Since both are Pareto optimal, they are preserved in the archive. Let $D = \{\bm{x} \in A \mid (H^+(\bm{x}) \cap A = \emptyset \wedge H^+(\bm{x}) \cap S_I^* \neq \emptyset) \vee (H^-(\bm{x}) \cap A = \emptyset \wedge H^-(\bm{x}) \cap S_I^* \neq \emptyset)  \}$, where $H^+(\bm{x}):=\{\bmx' \mid |\bmx'|_1 = |\bmx|_1 +1\}$ and $H^-(\bm{x}):=\{\bmx' \mid |\bmx'|_1 = |\bmx|_1 -1\}$ denote the Hamming neighbors of $\bmx$ with one more 1-bit and one less 1-bit, respectively. Intuitively, $D$ denotes the set of solutions in $A$ whose Hamming neighbor is Pareto optimal but not contained by $A$. Note that the size of archive is no greater than $n-2k+3$, i.e., the size of Pareto front. Then, by choosing the archive $A$ as parent population (whose probability is $1/2$), selecting a solution $\bmx \in D$ (whose probability is at least $1/(n-2k+3)$), and flipping one of the 0-bits (or 1-bits) with probability $(n-|\bmx|_1/n)\cdot (1-1/n)^{n-1}$ (or $(|\bmx|_1/n)\cdot (1-1/n)^{n-1}$), a new point on $F_I^*$ can be found. Thus, in each generation, the probability of generating a new point on $F_I^*$ is at least $(1/2)\cdot (n-|\bmx|_1/n)\cdot (1-1/n)^{n-1}/(n-2k+3)\ge (n-|x|_1)/(2e(n-2k+3)n)$. Since all solutions corresponding to points on $F^*_I$ are non-dominated, at least one solution corresponding to the new point will be preserved in the archive. Then, by repeating the above procedure, the whole internal Pareto front can be found. Thus, the total expected number of generations for finding the whole $F_I^*$ is at most $\sum_{i=k+1}^{|\bmx|_1} 2e(n-2k+3)n/i + \sum_{i=k+1}^{n-|\bmx|_1} 2e(n-2k+3)n/i = O(n^2\log n)$.

Combining the three phases, the total expected number of generations is $O(1+\mu n^k +n^2\log n) = O(\mu n^k)$, where $k\ge 3$. Thus, the theorem holds.
\end{proof}

Comparing with the result in Theorem~\ref{thm:ojzj1}, the running time for SMS-EMOA solving \ojzj\ reduces from $n^{\Omega(n)}$ to $O(\mu n^k)$ if reusing the archive. Theorem~\ref{thm:ojzj2} reveals that when using a small population for exploration, reusing the archive allows the algorithm to revisit previously missed regions, namely finding a better-covered Pareto front.

\subsection{Archive Reuse Helps Exploration}\label{sec:Exploration}

In the last subsection, we have theoretically shown that reusing the archive helps re-find regions missed during early search stages, thus enhances Pareto coverage. Next, we show that reusing the archive can also benefit exploration, thereby finding regions relatively hard to reach. In Theorem~\ref{thm:archive_benchmark}, we prove that when using a small population size $\mu\le n-2k+1$ and an archive without reuse, the expected running time for SMS-EMOA solving \aojzj\ is $\Omega(n^k)$. For practical MOEAs, e.g., SMS-EMOA, diversity in the objective space always plays an important role in the population update process (i.e., among non-dominated solutions, those in crowded regions are more likely to be removed). This example shows a scenario where some non-dominated solutions are sparse in the solution space and helpful for exploration, but are crowded in the objective space. With a small population and no archive reuse, these useful solutions may not be preserved in the population, limiting the exploration.

\begin{theorem}\label{thm:archive_benchmark}
  For SMS-EMOA solving \aojzj\ with $k\ge 13$, $a\ge k/4$ and $n-2k = \Omega(n)$, if using a population size $\mu\le n-2k+1$ and an archive without reuse, then the expected running time for finding the Pareto front is $\Omega(n^k)$.
\end{theorem}

\begin{proof}
We first prove that $B_1$, which denotes the event that the two boundary solutions of $F_I^*$ (i.e., the solutions with $k$ and $n - k$ 1-bits) is found before a solution $\bmx \in  \{ \bmz \mid |\bmz|_1 = k-a \vee |\bmz|_1 = n-k+a\} \cup \{0^n, 1^n\}$, occurs with probability at least 
$(1-e^{-\Omega(n)})\cdot e^{-1/\Omega(n^{1/4})}$. Then conditioned on the occurrence of event $B_1$, we prove that $B_2$, which denotes the event that no duplicate objective vectors are maintained in the population before finding a solution $\bmx \in  \{ \bmz \mid |\bmz|_1 = k-a \vee |\bmz|_1 = n-k+a\} \cup \{0^n, 1^n\}$, occurs with probability at least $e^{-1/\Omega(n^{1/4})}$. Finally, conditioned on the occurrence of events $B_1$ and $B_2$, we show that the expected number of generations for solving \aojzj\ is $\Omega(n^k)$.


Let $B_1$ denotes the event that the two boundary solutions of $F_I^*$ (i.e., the solutions with $k$ and $n - k$ 1-bits) is found before a solution $\bmx \in  \{ \bmz \mid |\bmz|_1 = k-a \vee |\bmz|_1 = n-k+a\} \cup \{0^n, 1^n\}$. Then, we show that $B_1$ occurs with probability at least $(1-e^{-\Omega(n)})\cdot e^{-1/\Omega(n^{1/4})}$.
During population initialization, by the Chernoff bound and $n-2k = \Omega(n)$, all initial solutions have number of 1-bits between $k$ and $n-k$ with probability at least $(1- e^{-\Omega(n)})^\mu \ge 1-\mu e^{-\Omega(n)} =  1-e^{\Omega(n)}$, where the equality holds by $\mu \le n-2k+1$. Then, assume that all solutions in the initial population have number of 1-bits between $k$ and $n-k$. Since dominated solutions cannot be preserved in the population, solutions $\bmx \in  \{ \bmz \mid |\bmz|_1 = k-a \vee |\bmz|_1 = n-k+a\} \cup \{0^n, 1^n\}$ can only be generated from a solution with a number of 1-bits in the range $[k..n-k]$. In particular, $0^n$ can be generated by selecting a solution and flipping all 1-bits, whose probability is at most $1/n^k$. A solution $\bmx \in \{ \bmz \mid |\bmz|_1 = k-a\}$ can be generated by selecting a solution $\bm{y}$ such that $|\bmy|_1 = l \ge k$, and flipping $|\bmy| + a-k$ 1-bits, whose probability is at most 
\begin{align}
&\frac{\binom{l}{l + a-k}}{n^{l + a-k}} =\frac{l!}{(l+a-k)!(k-a)! }\cdot \frac{1}{n^{l+a-k}}\\
&\le \frac{l!a!}{(l+a-k)! k!\cdot  n^{l-k}} \cdot \frac{k!}{a!(k-a)!\cdot n^a}  \\
&= \frac{a!}{(l+a-k)\cdots (l+1-k)} \cdot \frac{\binom{l}{l-k}}{n^{l-k}} \cdot \frac{\binom{k}{a}}{n^a}\le \frac{\binom{k}{a}}{n^a}.
\end{align}
The analysis of finding $1^n$ and a solution $\bmy\in \{|\bmz|_1 = n-k+a\}$ follows similarly from the above. Thus the probability of finding a solution $\bmx \in  \{ \bmz \mid |\bmz|_1 = k-a \vee |\bmz|_1 = n-k+a\} \cup \{0^n, 1^n\}$ is at most $2\cdot (1/n^k + \binom{k}{a}/n^a) \le 4\binom{k}{a}/n^a$.
Let $\bmx_{\min}$ and $\bmx_{\max}$ denote the solutions in the population with the minimum and maximum number of 1-bits, respectively, and assume that $k \le |\bmx_{\min}|_1 \le |\bmx_{\max}|_1 \le n - k$.
In each generation, SMS-EMOA selects $\bmx_{\min}$ (or $\bmx_{\max}$) with probability at least $1/\mu$, generates a new solution $\bmx'$ with more 0-bit (or 1-bit) by flipping a single 1-bit (or 0-bit) with probability at least $(1/n)(1 - 1/n)^{n - 1} \ge 1/(en)$. The offspring $\bmx'$ is then preserved in the population if $k \le |\bmx'|_1 \le n - k$, since SMS-EMOA always preserves the boundary solutions that maximize $f_1$ and $f_2$, respectively. Thus, the conditional probability of generating a solution closer to the boundary solutions of $F_I^*$ (i.e., the solutions with $k$ and $n - k$ 1-bits) before finding a solution $\bmx \in  \{ \bmz \mid |\bmz|_1 = k-a \vee |\bmz|_1 = n-k+a\} \cup \{0^n, 1^n\}$ is at least
\begin{align}\label{eq:conditional_before_G}
    \frac{1/(e\mu n)}{1/(e\mu n) + 4\binom{k}{a}/n^a} = \frac{1}{1+4e\mu\binom{k}{a}/n^{a-1}}.
\end{align}
Then, by repeating the above procedure at most $n-2k$ times, two boundary solutions of $F_I^*$ can be found before finding $\bmx \in  \{ \bmz \mid |\bmz|_1 = k-a \vee |\bmz|_1 = n-k+a\} \cup \{0^n, 1^n\}$ with probability at least
\begin{align}\label{eq:prod_conditional_G}
    &\prod_{i=0}^{n-2k-1} \frac{1}{1+4e\mu\binom{k}{a}/n^{a-1}} \ge \exp(-(n-2k)\cdot \frac{4e\mu \binom{k}{a}}{n^{a-1}})\\
    &\ge \exp(-n^2\cdot \frac{4e\mu \cdot \binom{k}{\lfloor k/2\rfloor}}{n^{a}}) \ge \exp(-\frac{n^2 \cdot 4\sqrt{2}e^{13/12}\mu\cdot 2^k }{\sqrt{\pi (k-1)}n^a})\\
    &\ge \exp(-n^2\cdot \frac{4\sqrt{2}e^{13/12}\mu}{\sqrt{\pi (k-1)} (n/16)^{k/4}}) = e^{-1/\Omega(n^{1/4})},
\end{align}
where the first inequality holds by $1+b\le e^b$ for any $b\in \mathbb{R}$, the second inequality holds by $\binom{k}{a}\le \binom{k}{\lfloor k/2\rfloor}$ for all $a\in [0..k]$, the third inequality holds by $\binom{k}{\lfloor k/2\rfloor}\le 2^k\sqrt{2e^{1/6}/\pi (k-1)}$ as analyzed in Eq.~\eqref{eq:binom_stiling1} and~\eqref{eq:binom_stiling2}, the fourth inequality holds by $a\ge k/4$, and the last equality holds by $k\ge 13$ and $\mu \le n-2k+1$. Since the probability that all initial solutions belong to $S_I^*$ is $1-e^{-\Omega(n)}$, event $B_1$ occurs with probability at least $(1-e^{-\Omega(n)})\cdot e^{-1/\Omega(n^{1/4})}$.

Conditioned on the occurrence of $B_1$, let $B_2$ denote the event that no duplicate objective vectors are maintained in the population before finding a solution $\bmx \in  \{ \bmz \mid |\bmz|_1 = k-a \vee |\bmz|_1 = n-k+a\} \cup \{0^n, 1^n\}$. Then, we show that $B_2$ occurs with probability at least $e^{-1/\Omega(n^{1/4})}$.
Since $B_1$ occurs, two boundary points of $F_I^*$ have been found and all solutions in the population belong to $S_I^*$. Because solutions with duplicate objective vectors have a hypervolume contribution $\Delta_{\bmr} = 0$, they are removed from the population before non-dominated solutions with unique objective vectors. Thus, once the algorithm finds a solution corresponding to a unique point on the Pareto front, it will always be preserved in the population until all duplicates are removed. Let $D = \{\bm{x} \in P \mid (H^+(\bm{x}) \cap P = \emptyset \wedge H^+(\bm{x}) \cap S_I^* \neq \emptyset) \vee (H^-(\bm{x}) \cap P = \emptyset \wedge H^-(\bm{x}) \cap S_I^* \neq \emptyset)  \}$, where $H^+(\bm{x}):=\{\bmx' \mid |\bmx'|_1 = |\bmx|_1 +1\}$ and $H^-(\bm{x}):=\{\bmx' \mid |\bmx'|_1 = |\bmx|_1 -1\}$ denote the Hamming neighbors of $\bmx$ with one more 1-bit and one less 1-bit, respectively. Intuitively, $D$ denotes the set of solutions in $P$ whose Hamming neighbor is belong to $S_I^*$ but not contained by $P$. Then a new unique point on the Pareto front can be generated by selecting a solution in $D$, and flipping a single 1-bit or 0-bit with probability at least $(1/\mu)\cdot (1/n)(1-1/n)^{n-1} \ge 1/(e\mu n)$. Then, the analysis proceeds similarly to that of the probability of event $B_1$, the probability of finding a solution $\bmx \in  \{ \bmz \mid |\bmz|_1 = k-a \vee |\bmz|_1 = n-k+a\} \cup \{0^n, 1^n\}$ with probability at most $4\binom{k}{a}/n^a$. Thus, the conditional probability is the same as in Eq.~\eqref{eq:conditional_before_G}, i.e., $1/(1+4e\mu\binom{k}{a}/n^{a-1})$. Since the population size $\mu \le n - 2k + 1 = |F_I^*|$, by repeating the above procedure at most $\mu - 2 \le n - 2k-1$ times, no duplicate objective vectors are maintained in the population. Then, similar to the analysis in Eq.~\eqref{eq:prod_conditional_G}, event $B_2$ occurs with probability at least $e^{-1/\Omega(n^{1/4})}$.

Finally, we show that conditioned on the occurrence of events $B_1$ and $B_2$, the expected number of generations for solving \aojzj\ is $\Omega(n^k)$. We first show that after events $B_1$ and $B_2$ occur, solutions $\bmx \in  \{ \bmz \mid |\bmz|_1 = k-a \vee |\bmz|_1 = n-k+a\}$ can't be preserved in the population before finding $1^n$ or $0^n$. Suppose that in some generation, the algorithm generates a solution $\bmx'$ such that $\bmf(\bmx') = (2k+1/n,n-1/n)$ (i.e., $|\bmx'|_1 = k-a$). Then, for any objective vector $\bmv = (b,n+2k-b)$ with $b\in [2k+1..n-1]$, the region covered by $\bmv$ is larger than the hyper-rectangle defined by
\begin{align}
\{ \bm{f} \in \mathbb{R}^2 \mid &(b - 1 + 1/n < f_1 \le b) \\ &\wedge (n + 2k - b - 1 < f_2 \le n + 2k - b) \},
\end{align}
whose area is $(b-b+1-1/n)\cdot (n+2k-b-(n+2k-b-1)) = 1-1/n$. Specifically, this hyper-rectangle can only be achieved when $\bmv=(k+1,n-1)$ with neighboring points $(2k+1/n,n-1/n)$ and $(k+2,n-2)$. And for objective vector $\bmf(\bmx') = (2k+1/n,n-1/n)$, the region covered by $\bmf(\bmx')$ is less than the hyper-rectangle defined by
\begin{align}
    \{ \bm{f} \in \mathbb{R}^2 \mid &(2k < f_1 \le 2k+1/n) \\&\wedge (2k < f_2 \le n-1/n) \},
\end{align}
whose area is $(2k-2k+1/n)\cdot (n-1/n-2k) = (n-1/n-2k)/n < 1-2k/n$. After the events $B_1$ and $B_2$ occurs, there are no duplicate objective vectors in the population, and all solutions belong to $S_I^*$. Hence, for any solution $\bmz \in P$ with $|\bmz|_1 \in [k+1..n-k-1]$, its hypervolume contribution satisfies $\Delta_{\bm{r}}(\bmz, R_1) \ge 1 - \frac{1}{n} > 1 - \frac{2k}{n} > \Delta(\bmx', R_1)$. 
In addition, SMS-EMOA always preserves the boundary solutions $\bmz$ with $|\bmz|_1 \in \{k, n - k\}$. Therefore, $\bmx'$ cannot be preserved in the population before the algorithm finds $0^n$ or $1^n$. The same argument applies to $\bmy'$ with $|\bmy|_1 = n - k + a$. Consequently, $0^n$ and $1^n$ can only be generated by selecting a solution $\bmx \in S_I^*$ and flipping at least $k$ 1-bits or 0-bits, which occurs with probability at most $1/n^k$. Thus, the expected number of generations for finding $0^n$ or $1^n$ is at least $\Omega(n^k)$.

Combining the above analysis, the expected number of generations for SMS-EMOA solving \aojzj\ is at least $(1-e^{-\Omega(n)})\cdot (e^{-1/\Omega(n^{1/4})})^2 \cdot \Omega(n^k) = \Omega(n^k)$. Thus, the theorem holds.
\end{proof}

Through the analysis of Theorem~\ref{thm:archive_benchmark}, we observe that using a small population size makes it difficult to retain solutions with objective vectors $(2k+1/n, n-1/n)$ or $(n-1/n, 2k+1/n)$, due to their small $\Delta_{\bmr}$ values. This limits the exploration ability to find $0^n$ and $1^n$. 
In Theorem~\ref{thm:reusing_benchmark}, we prove that if reusing the archive, the expected number of generations for SMS-EMOA solving \aojzj\ is $O(\max\{\mu n^a/k, n^{k-a+1}, n^2\log n\})$.

\begin{theorem}\label{thm:reusing_benchmark}
  For SMS-EMOA solving \aojzj\ with $n-2k = \Omega(n)$, if using a population size $\mu\ge 2$ and reusing the archive, then the expected running time for finding the Pareto front is $O(\max\{ \mu n^a/k, n^{k-a+1},n^2\log n\})$.
\end{theorem}

\begin{proof}
We divide the running process into four phases. The first phase starts after initialization and finishes until one point on $F_I^*$ (i.e., the objective vector $\bmf(\bmx)$ such that $|\bmx|_1\in [k..n-k]$) is found; the second phase starts after the first phase and finishes until one solution $\bmx\in \{1^n\}\cup \{\bmz\mid |\bmz|_1 =n-k+a\}$ and one solution $\bmy\in \{0^n\}\cup \{\bmz\mid |\bmz|_1 = k-a\}$ are found; the third phase starts after the second phase and finishes until $0^n$, $1^n$ and all points on $G$ are found; the fourth phase starts after the third phase and finishes until the whole internal Pareto front $F_I^*$ is found.

For the first phase, the analysis is the same as that of first phase of Theorem~\ref{thm:ojzj2}. Thus, the expected number of generations for finding one point on $F^*_I$ is at most $O(1)$.

For the second phase, we show that the expected number of generations for finding a solution with $k-a$ 1-bits is $O(\mu n^{a})$ (i.e., objective vector $(2k+1/n, n-1/n)\in G$), and by symmetry, the same bound holds for finding a solution $\bmy$ with $n-k+a$. We assume that no solution $\bm{x} \in \{0^n\} \cup \{\bm{z} \mid |\bm{z}|_1 = k - a\}$ has been found after the first phase; otherwise, the second phase finishes. Since SMS-EMOA directly preserves the two boundary points, i.e., the two solutions with the maximum $f_1$ and $f_2$ values respectively. Thus, the maximum $f_2$ value in $P\cup {\bmx'}$ will not decrease, where $\bmx'$ is the offspring solution generated in each generation. Then similar to the analysis in Eq.~\eqref{eq:ojzj1}, the probability of generating a solution with more 0-bits in each generation is at least $|\bmx_{\min}|/2e\mu n$, where $\bmx_{\min} = \arg\min_{\bmx\in P} |\bmx|_1$. Thus, the expected number of generations for increasing the maximum $f_2$ value to $n$, i.e., finding a solution with $k$ 1-bits, is at most $\sum_{i=k+1}^{n-k} 2e\mu n/i = O(\mu n \log n)$. Then, a solution $\bmx$ with $k$ 1-bits remains preserved in the population $P$ until a solution $\bmx'\in \{0^n\} \cup \{\bm{z} \mid |\bm{z}|_1 = k - a\}$ is covered. Now, we consider finding $\bmx'$ with $k-a$ 1-bits. In each generation, SMS-EMOA selects $\bmx$ with $k$ 1-bits from $P$ with probability $1/\mu$, and generates $\bmx'$ by flipping $a$ 1-bits with probability $(\binom{k}{a}/n)(1-1/n)^{n-a}\ge k/(en^a)$. Then the probability of generating $\bmx'$ in each generation is at least $(1/\mu) \cdot (k/en^a) = k/(e\mu n^a)$. Thus, the expected number of generations for finding $\bmx'$, i.e., finishing the second phase, is at most $e\mu n^a/k = O(\mu n^a/k)$.  

For the third phase, we show that the expected number of generation for finding $0^n$ and the objective vector $(2k+1/n, n-1/n) \in G$ is $O(n^{k-a+1})$, and by symmetry, the same bound holds for finding $1^n$ and $(n-1/n, 2k+1/n)\in G$. After the second phase, either $0^n$ or $(2k+1/n, n-1/n)$ is found. Thus, we separately consider the running time of generating $0^n$ from $(2k+1/n, n - 1/n)$, and the running time of generating $(2k+1/n, n - 1/n)$ from $0^n$. In each generation, $0^n$ can be generating by selecting $\bmx$ with $k-a$ 1-bits in the archive $A$ (whose probability is $(1/2)\cdot (1/|A|)$), and flipping $k-a$ 1-bits with probability at least $(1/n^{k-a})(1-1/n)^{n-(k-a)}\ge 1/(en^{k-a})$. Since the size of archive $A$ is no greater than $n-2k+5$, i.e., the size of Pareto front. The probability of generating $0^n$ from a solution with $n-a$ 1-bits in each generation is at least $(1/2|A|) \cdot 1/(en^{k-a}) \ge 1/(2en^{k-a+1})$. Thus the expected number of generations for finding $0^n$ is at most $2en^{k-a+1} = O(n^{k-a+1})$. For the case of generating $(2k+1/n, n - 1/n)$ from $0^n$, the analysis is similar to that presented above. Thus, the expected number of generations for the third phase is $O(n^{k-a+1})$.

For the fourth phase, the analysis is the same as that of third phase of Theorem~\ref{thm:ojzj2}. Thus, the expected number of generations for finding the whole internal Pareto front $F^*_I$ is $O(n^2\log n)$.

Combining the four phases, the total expected number of generations is $O(1) + O(\mu n^{a}/k) + O(n^{k-a+1}) + O(n^2\log n) = O(\max\{ \mu n^a/k, n^{k-a+1}, n^2\log n\})$. Thus, the theorem holds.
\end{proof}

Comparing with the result in Theorem~\ref{thm:archive_benchmark}, when reusing the archive, the expected running time for SMS-EMOA solving \aojzj\ reduces from $\Omega(n^k)$ to $O(\max\{\mu n^a/k, n^{k-a+1},n^2 \log n\})$. Theorem~\ref{thm:reusing_benchmark} reveals that in scenarios where the distribution of solutions in decision space and objective space is inconsistent, using a small population may lead to the loss of some solutions helpful for exploration, while reusing the archive enables revisiting these helpful solutions, thereby enhancing exploration.

\section{Reusing Archive vs. Using A Large Population}

In the last section, we theoretically show that reusing the archive is beneficial compared to using an archive without reuse. Now, we show that reusing the archive with a small population can be better than directly using a large population. Specifically, we further show that when reusing the archive, the upper bound on the expected running time for SMS-EMOA solving \aojzj, \ojzj, \omm, and \lotz\ is lower than that of using a large population size. We prove in Theorem~\ref{thm:population_benchmark} that the expected running time for SMS-EMOA solving \aojzj\ is $O(\mu \cdot\max\{n^a/k,n^{k-a}\})$ if using a population size $\mu \ge n-2k+5$. The proof idea is to divide the optimization procedure into three phases, where the first phase aims at finding the whole internal Pareto front $F_I^*$, the second phase aims at finding the two points on $G$, and the third phase aims at finding $1^n$ and $0^n$. In Theorem~\ref{thm:ojzj3}, Bian et al.~\shortcite{bian23stochastic} showed that the expected running time for SMS-EMOA solving \ojzj\ is $O(\mu n^k)$ if using a population size $\mu \ge n-2k+3$.

\begin{theorem}\label{thm:population_benchmark}
  For SMS-EMOA solving \aojzj\ with $n-2k = \Omega(n)$, if using a population size $\mu\ge n-2k+5$, then the expected running time for finding the Pareto front is $O(\mu \cdot \max\{ n^{a}/k, n^{k-a} \})$.
\end{theorem}

\begin{proof}
We divide the running process into three phases. The first phase starts after initialization and finishes until the internal Pareto front $F_I^*$ is found; the second phase starts after the first phase and finishes when the two points on $G$ are both found; the third phase starts after the second phase and finishes until $1^n$ and $0^n$ are both found.

Firstly, we show that once an objective vector $\bm{f}^*$ on the Pareto front has been found, it will always be maintained in the population. Because only one solution is removed in each generation by Algorithm~\ref{alg:sms}, we only need to consider the case that exactly one solution (denoted as $\bm{x}^*$) corresponds to $\bm{f}^*$. Since $\bm{x}^*$ is the Pareto optimal solution, we have $\bm{x}^* \in R_1$ in the non-dominated sorting procedure. Moreover, We have $\Delta_{\bm{r}}(\bm{x}^*,R_1)=HV_{\bm{r}}(R_1)-HV_{\bm{r}}(R_1\setminus \{\bm{x}^*\})>0$ because $\bm{f}^*$ can't be covered by any objective vector in $\bm{f}(\{0,1\}^n) \backslash \bm{f}^*$. If at least two solutions in $R_1$ have the same objective vector, they must have a zero $\Delta$-value, because removing one of them will not decrease the hypervolume covered. Thus, there exist at most $n-2k+5$ solutions (i.e., the size of Pareto front) in $R_1$ with $\Delta>0$. Note that the removed solution is from $P \cup \{\bm{x}'\}$, which contains at least $n-2k+6$ solutions, as $|P| \geq n-2k+5$. Thus, $\bm{x}^*$ will not be removed because it is one of the best $n-2k+5$ solutions. Then, for the first phase, the analysis is similar to that of the first and third phase of Theorem~\ref{thm:ojzj2}. The expected number of generations for finding one point on $F_I^*$ is $O(1)$ (analogous to the first phase of Theorem~\ref{thm:ojzj2}).
Subsequently, considering selecting solution from the population (whose probability is $1/\mu$) rather than from the archive, the expected number of generations for finding the whole $F^*_I$ is $O(\mu n \log n)$ (analogous to the third phase of Theorem~\ref{thm:ojzj2}). Thus, the expected number of generations for the first phase is $O(1 + \mu n\log n) = O(\mu n\log n)$.

For the second phase, we show that the expected number of generations for finding $(2k+1/n, n-1/n)$ on $G$ is $O(\mu n^{a}/k)$, and by symmetry, the same bound holds for finding $(n-1/n,2k+1/n)$. A solution $\bm{x}$ with $k$ 1-bits is found and preserved in the population after the first phase. In each generation, SMS-EMOA selects $\bm{x}$ from $P$ with probability $1/\mu$, and generates a solution $\bm{x}'$ with $k-a$ 1-bits such that $\bm{f}(\bm{x}') = (2k+1/n, n-1/n)$ by flipping $a$ 1-bits with probability $(\binom{k}{a}/n^{a})(1-1/n)^{n-a} \ge k/(en^a)$. The probability of these events happening in a single generation is at least $(1/\mu)\cdot(k/en^a) = k/(e\mu n^a)$. Thus, the expected number of generations for finding $(2k+1/n, n-1/n)$ is at most $e\mu n^a/k = O(\mu n^a/k)$.

For the third phase, we show that the expected number of generations for finding $0^n$ is $O(\mu n^{k-a})$, and by symmetry, the same bound holds for finding $1^n$. A solution $\bm{x}$ with $k-a$ 1-bits is found after the second phase. In each generation, SMS-EMOA selects $\bm{x}$ from $P$ with probability $1/\mu$, and generates a solution $0^n$ by flipping $k-a$ 1-bits with probability $(1/n^{k-a})(1-1/n)^{n-(k-a)} \ge 1/(en^{k-a})$. The probability of these events happening in a single generation is at least $(1/\mu)\cdot(1/en^{k-a}) = 1/(e\mu n^{k-a})$. Thus, the expected number of generations for finding $0^n$ is at most $e\mu n^{k-a} = O(\mu n^{k-a})$.

Combining the three phases, the total expected number of generations is $O(\mu n\log n + \mu n^a/k + \mu n^{k-a}) = O(\mu \cdot \max\{ n^{a}/k, n^{k-a} \})$. Thus, the theorem holds.
\end{proof}

\begin{theorem}[Bian et al.~\shortcite{bian23stochastic}]\label{thm:ojzj3}
    For SMS-EMOA solving \ojzj\ with $n-2k = \Omega(n)$, if using a population size $\mu\ge n-2k+3$, then the expected running time for finding the Pareto front is $O(\mu n^k)$.
\end{theorem}
 
Comparing the results in Theorems~\ref{thm:reusing_benchmark} and~\ref{thm:population_benchmark}, as well as those in~\ref{thm:ojzj2} and~\ref{thm:ojzj3}, the upper bound on the expected running time for SMS-EMOA solving \aojzj\ and \ojzj\ can be reduced by a factor of $\Theta(n)$. The reason is that, with archive reuse, the population size can be constant, rather than the size of the Pareto front (i.e., $n - 2k + 5$ or $n-2k+3$). On \aojzj, this improvement becomes evident only when $a > k/2$.


Bian et al.~\shortcite{bian2024archive} showed that for SMS-EMOA solving \omm\ and \lotz, if using a small population and an archive without reuse, the upper bound on the expected running time can be reduced by a factor of $\Theta(n)$, compared to using a large population size $\mu\ge n+1$ (i.e., the size of Pareto front). Next, we will show that archive reuse does not diminish this advantage. 
Note that Bian et al.~\shortcite{bian2024archive} considered one-point crossover, so the archive reuse method is adapted to select each parent solution for recombination with probability $1/2$ from the population and $1/2$ from the archive.
Then, we only need to consider the case where both parents are selected from the population (whose probability is $1/4$), and the remaining analysis is the same as that in~\cite{bian2024archive}. The definition of \omm\ and \lotz, and full proofs are given in the supplementary.

\begin{theorem}[Zheng et al.~\shortcite{zheng2024sms}]\label{thm:omm}
    Consider using the SMS-EMOA with $\mu\ge n+1$ to optimize the \omm\ problem. Then after at most $2e\mu n(\ln n+1)$ iterations in expectation, the Pareto front is covered.
\end{theorem}

\begin{theorem}\label{thm:sms-arc-omm}
	For SMS-EMOA solving \omm, if using a population size $\mu\ge 2$, one-point crossover and reusing the archive, then the expected running time for finding the Pareto front is $O(\mu n \log n)$.
\end{theorem}

\begin{theorem}[Zheng et al.~\shortcite{zheng2024sms}]\label{thm:lotz}
    Consider using the SMS-EMOA with $\mu\ge n+1$ to optimize the \lotz\ problem. Then after at most $2e\mu n^2$ iterations in expectation, the Pareto front is covered.
\end{theorem}

\begin{theorem}\label{thm:sms-arc-lotz}
	For SMS-EMOA solving \lotz, if using a population size $\mu\ge 2$, one-point crossover and reusing the archive, then the expected running time for finding the Pareto front is $O(\mu n^2+\mu^2 n\log n)$.
\end{theorem}

The expected running time for SMS-EMOA (without using an archive) solving \omm\ and \lotz\ has been shown to be $O(\mu n\log n)$ and $O(\mu n^2)$ in Theorems~\ref{thm:omm} and~\ref{thm:lotz}, respectively, where the population size $\mu\ge n+1$~\cite{zheng2024sms}. Thus, our results in Theorems~\ref{thm:sms-arc-omm} and~\ref{thm:sms-arc-lotz} show that if a constant population size is used for SMS-EMOA reusing the archive, the expected running time can be reduced by a factor of $\Theta(n)$. 


\begin{figure*}[!t]
    \begin{flushright}
    \vspace{-10pt}
    \end{flushright}
	\renewcommand{\arraystretch}{0.1} 
	\begin{center}
         \begin{tabular}{@{}c@{}@{}c@{}@{}c@{}}
            \includegraphics[scale=0.28]{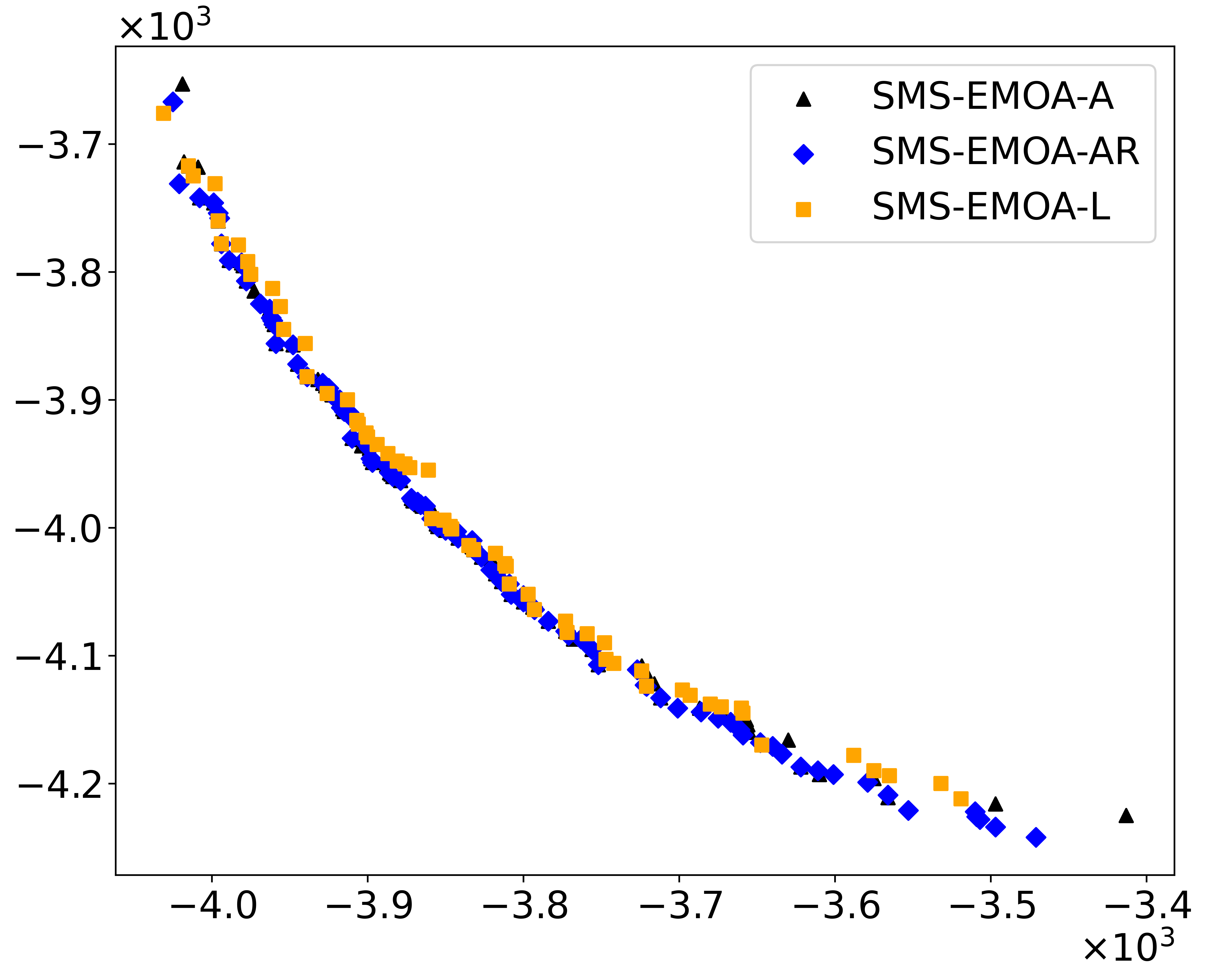} &
            \includegraphics[scale=0.28]{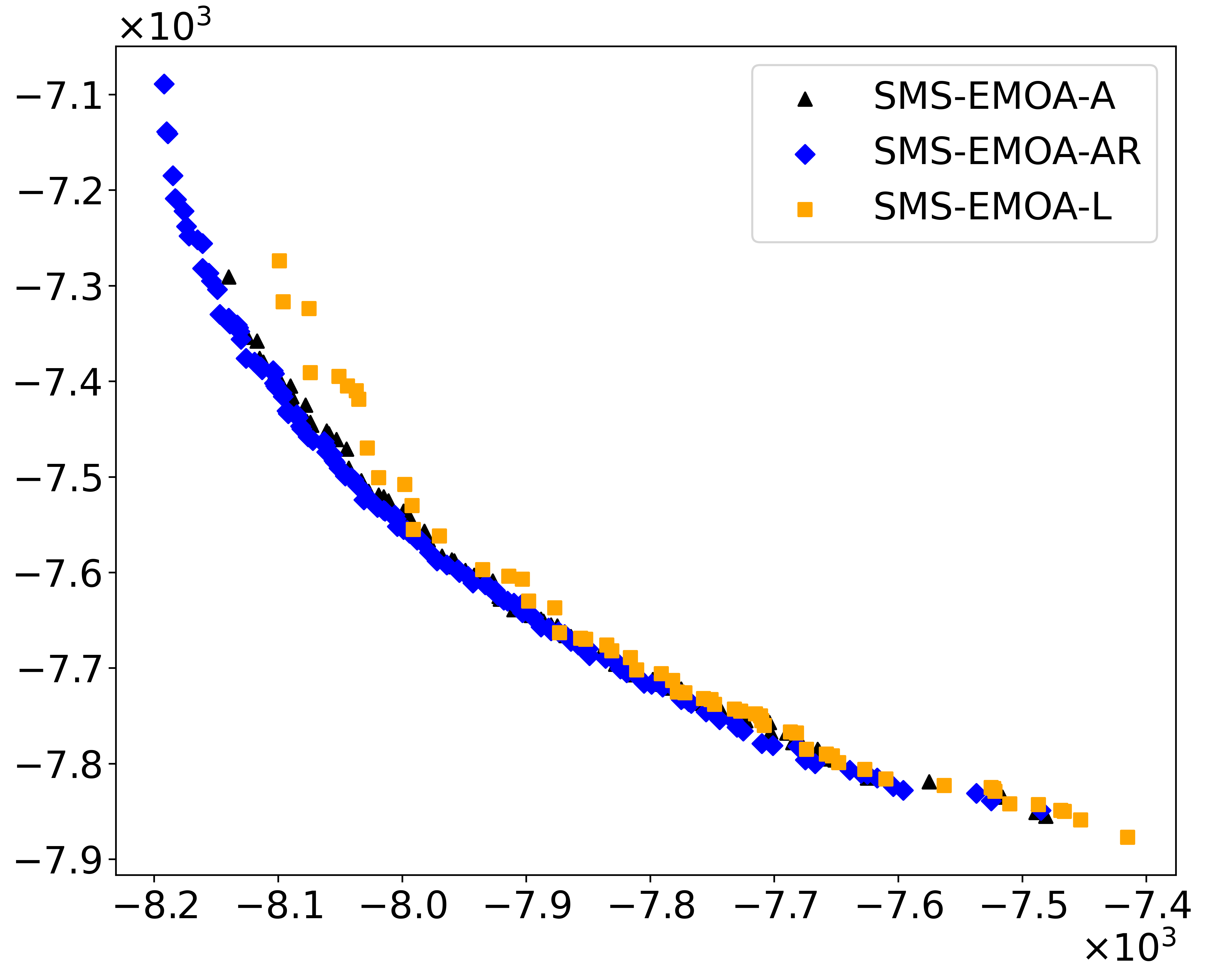} &
            \includegraphics[scale=0.28]{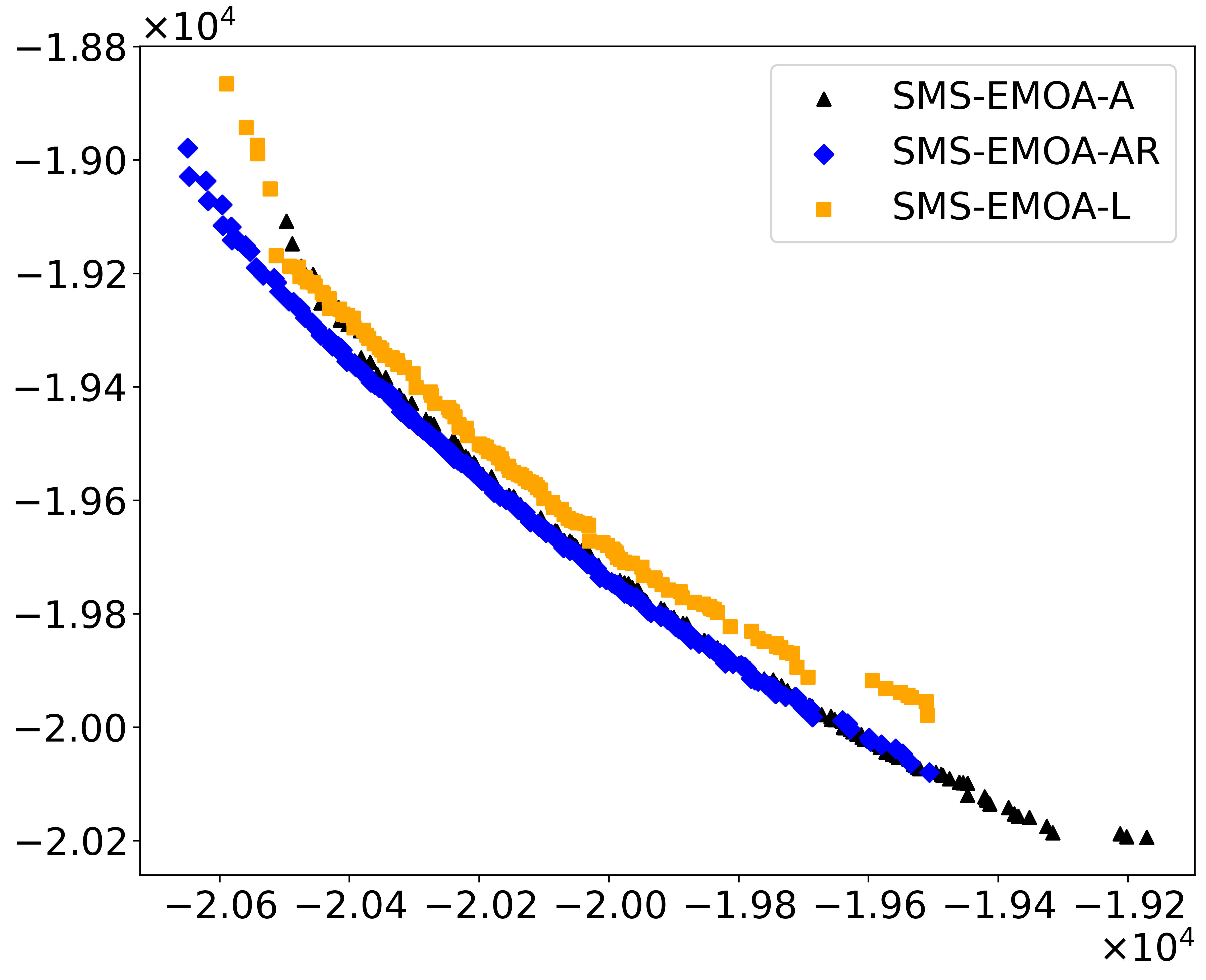} \\
            (a) KP 100D & (b) KP 200D & (c) KP 500D \\[6pt]

            \includegraphics[scale=0.28]{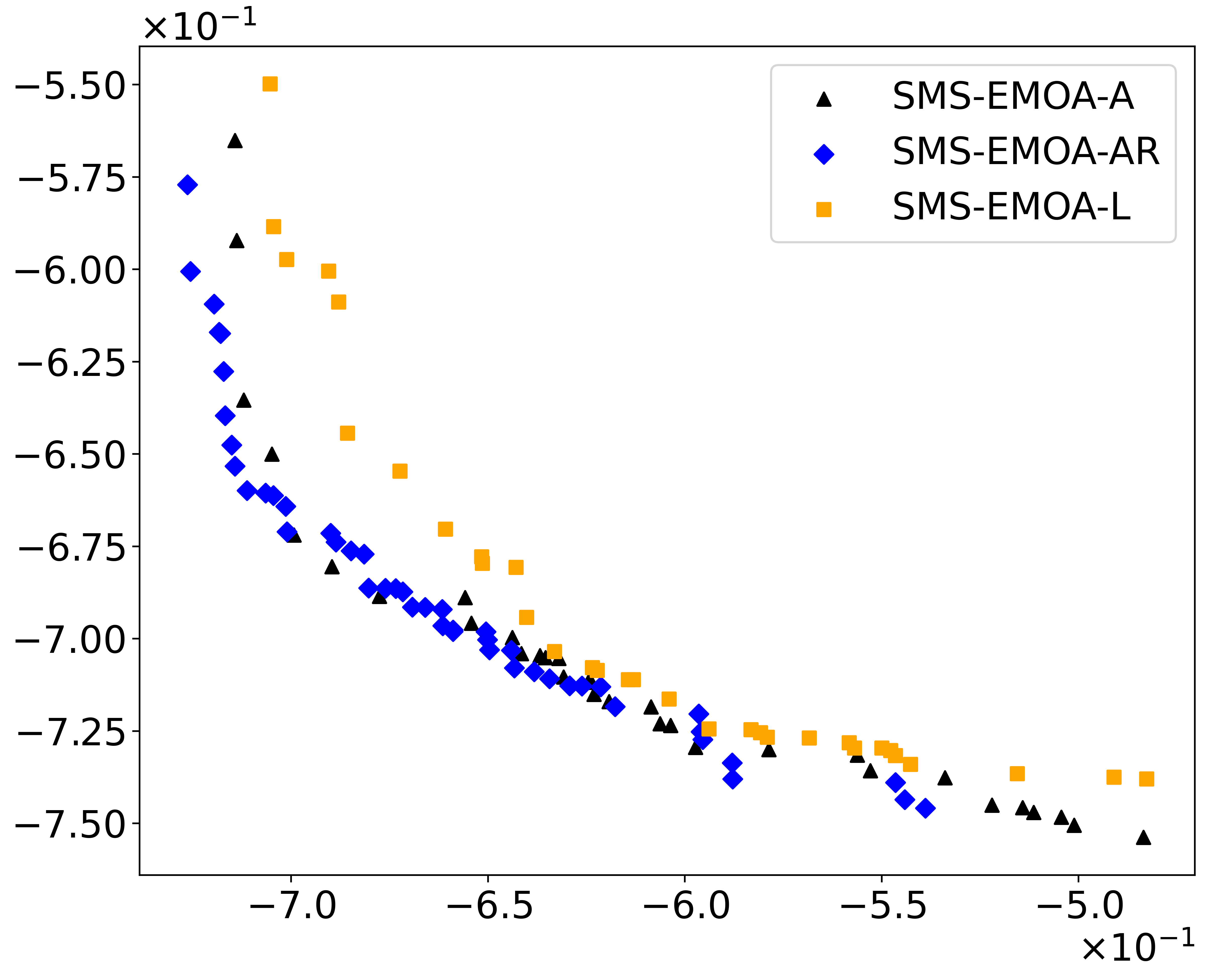} &
            \includegraphics[scale=0.28]{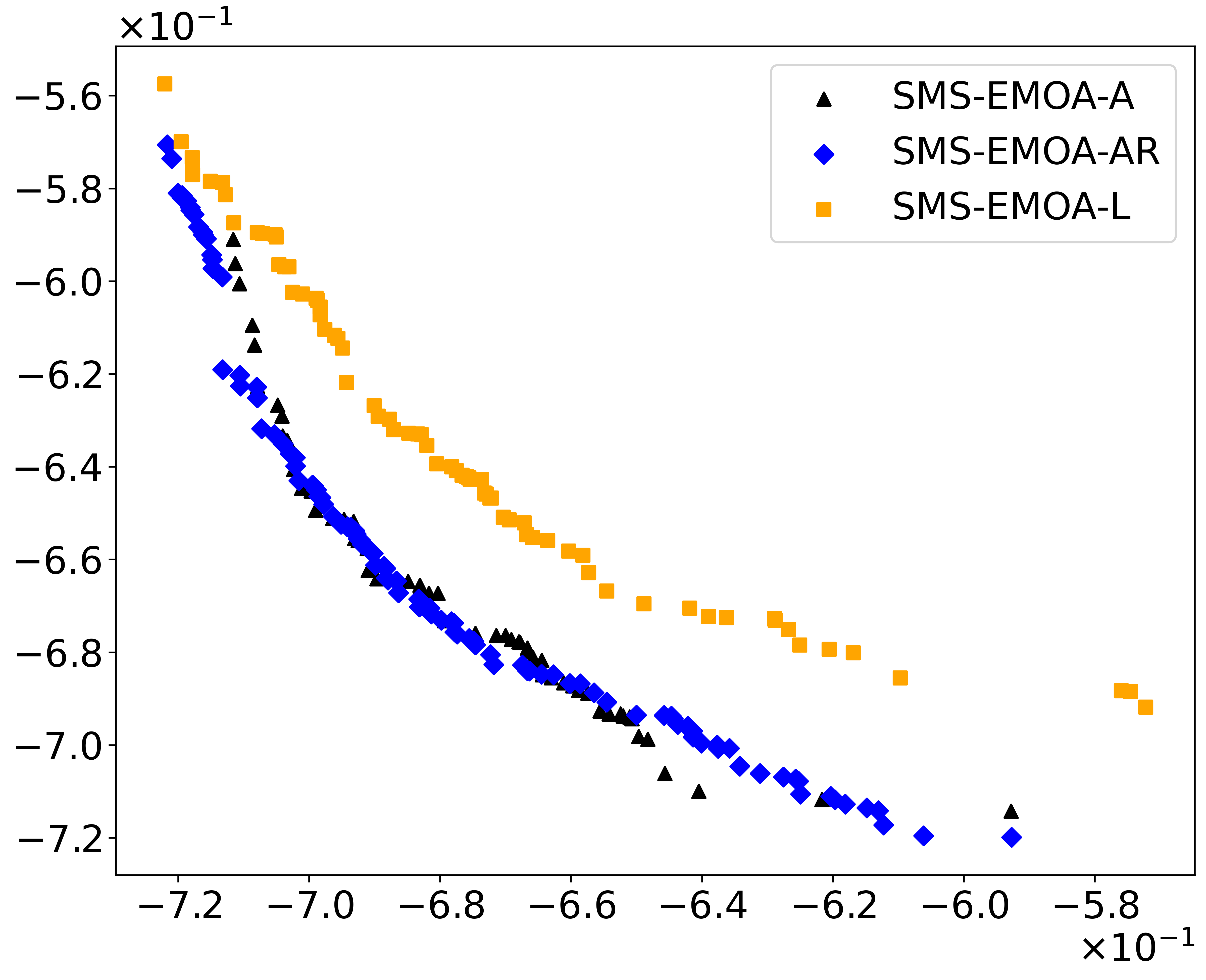} &
            \includegraphics[scale=0.28]{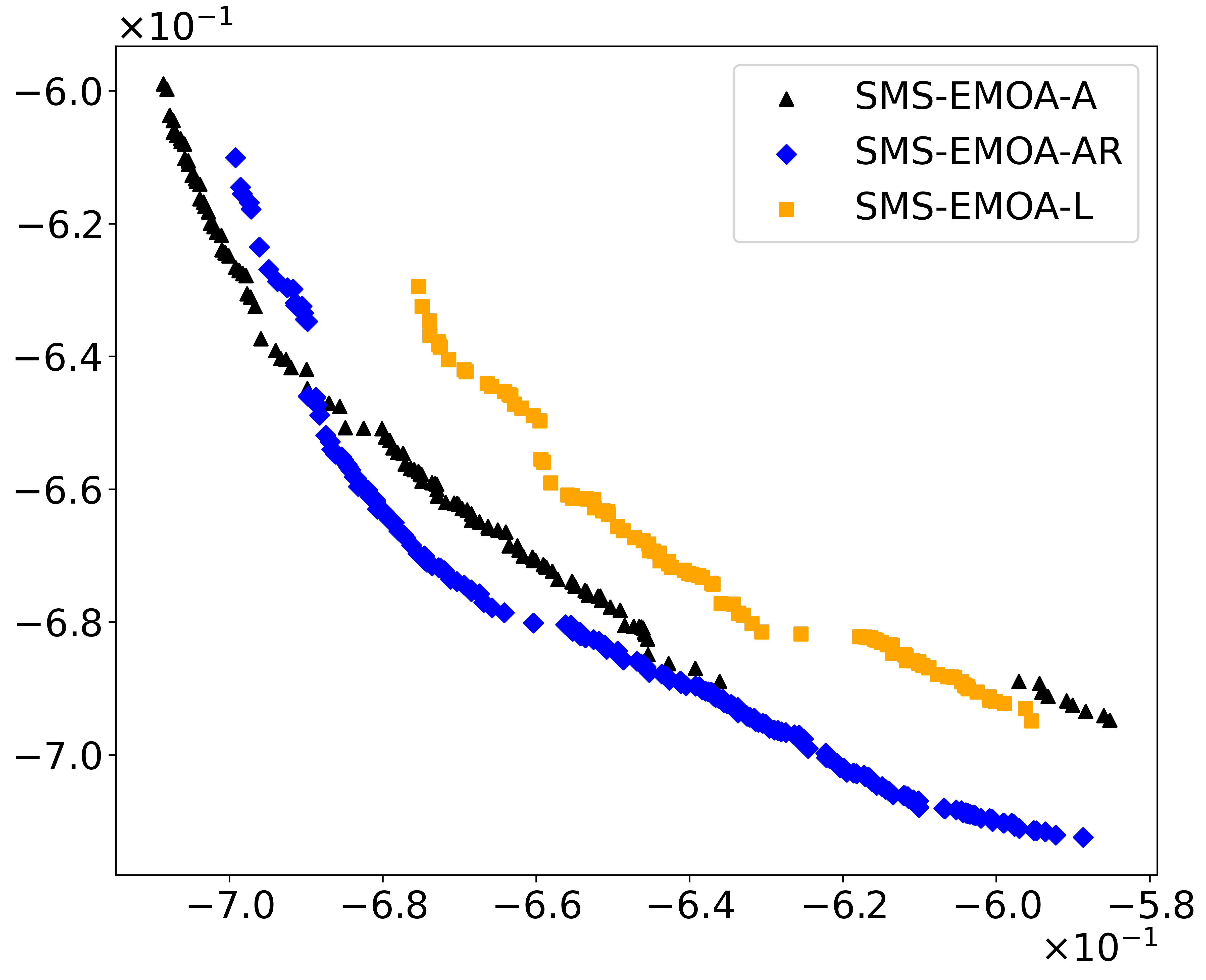} \\
            (d) NK 100D & (e) NK 200D & (f) NK 500D \\[6pt]

            \includegraphics[scale=0.28]{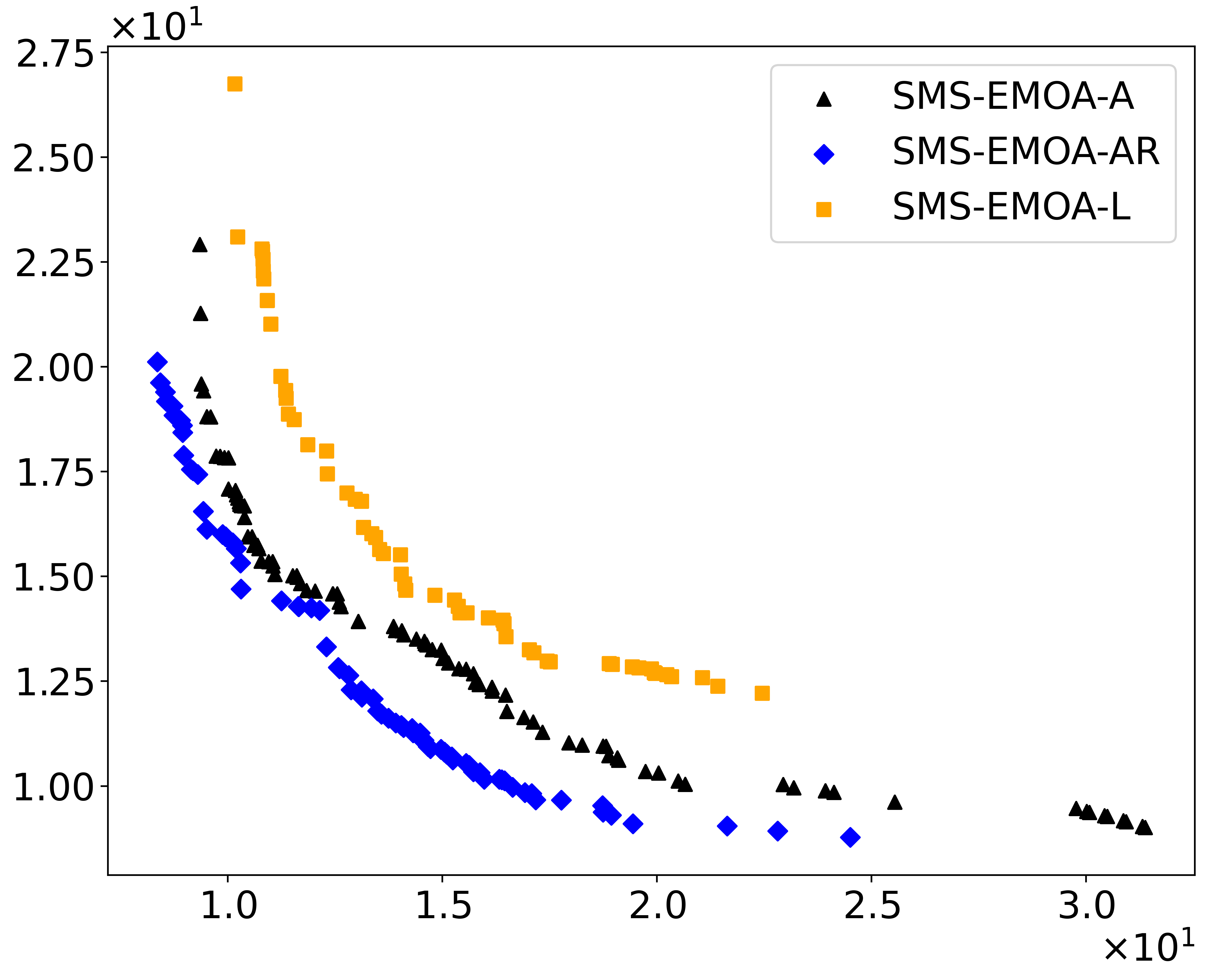} &
            \includegraphics[scale=0.28]{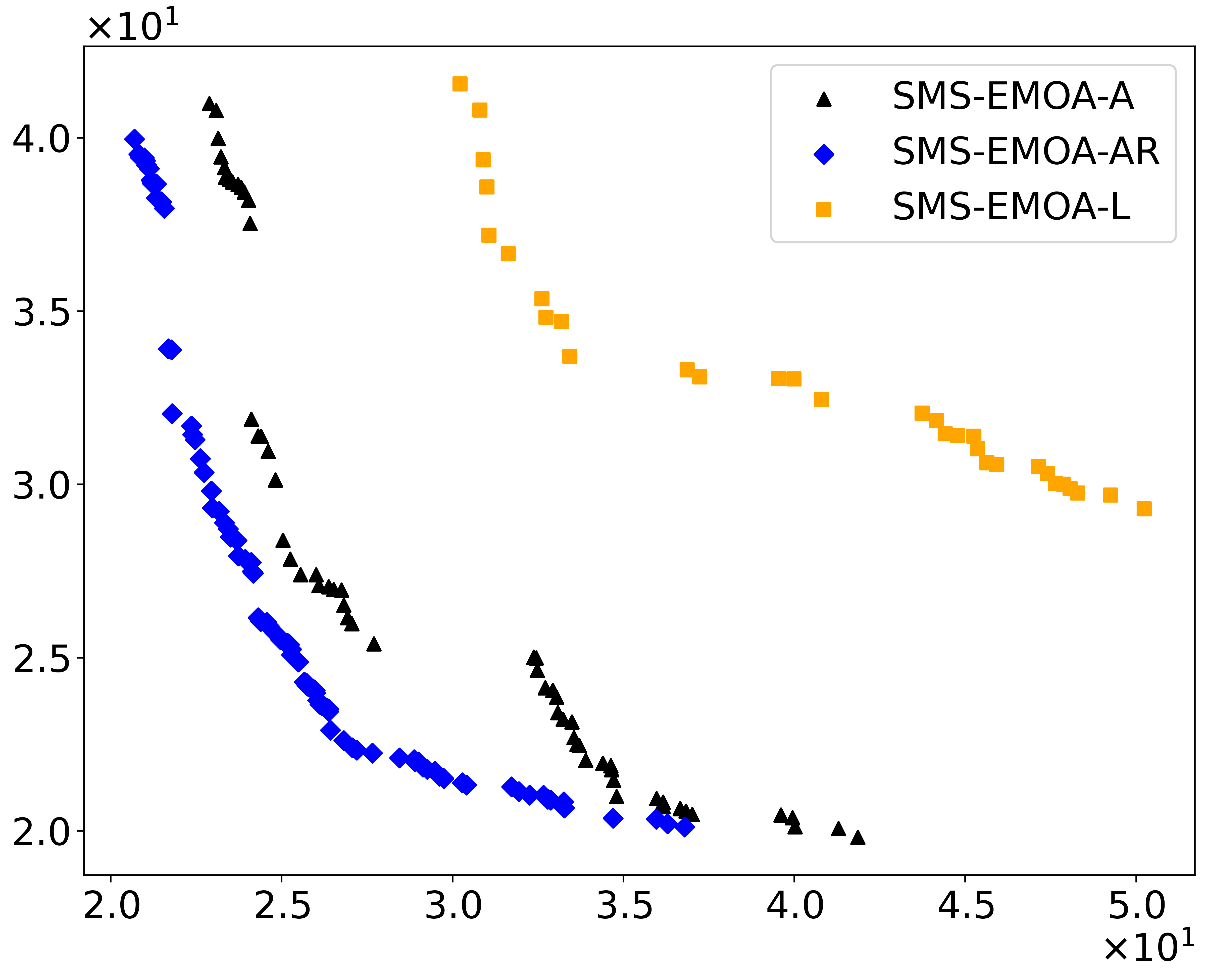} &
            \includegraphics[scale=0.28]{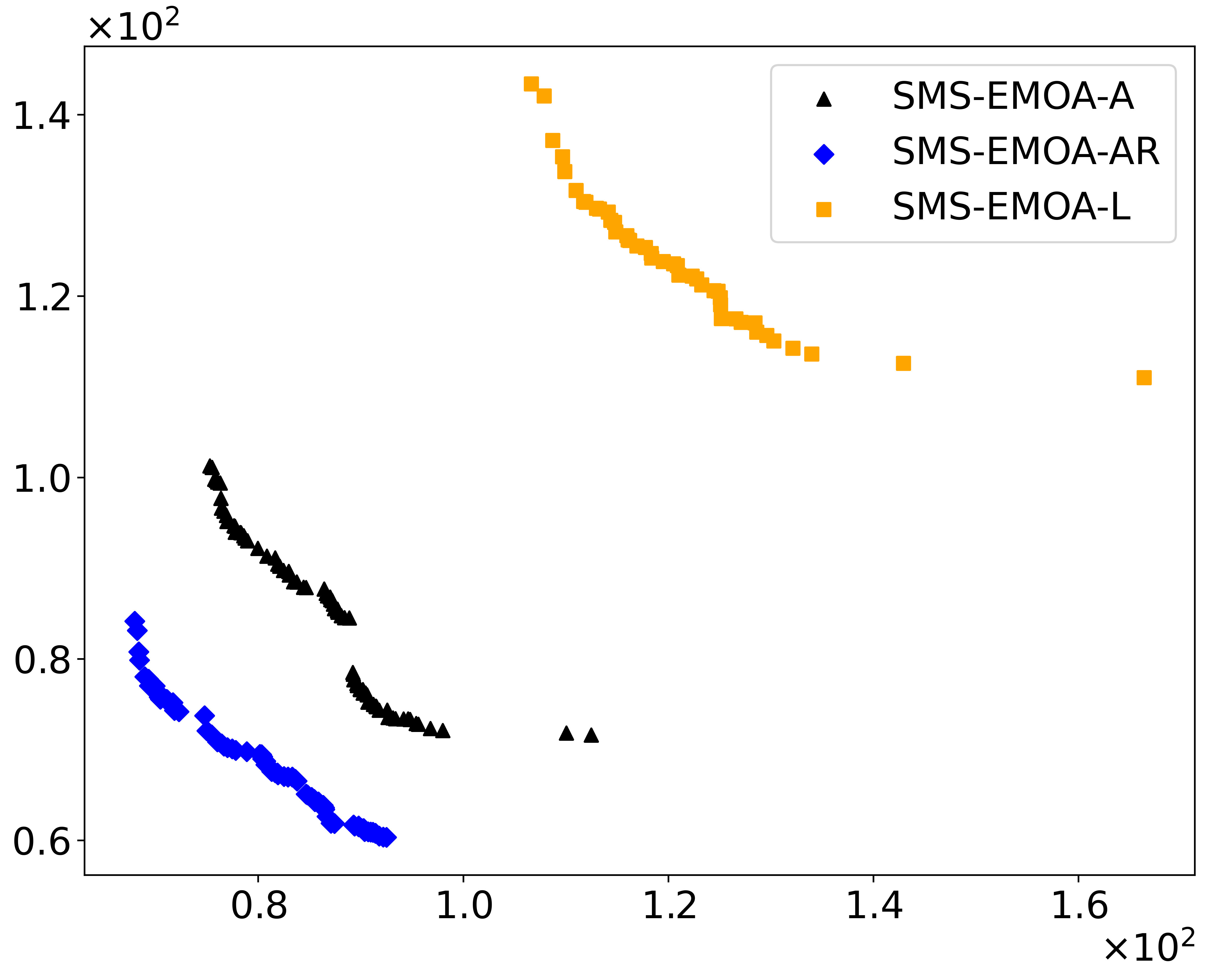} \\
            (g) TSP 100D & (h) TSP 200D & (i) TSP 500D \\[6pt]

            \includegraphics[scale=0.28]{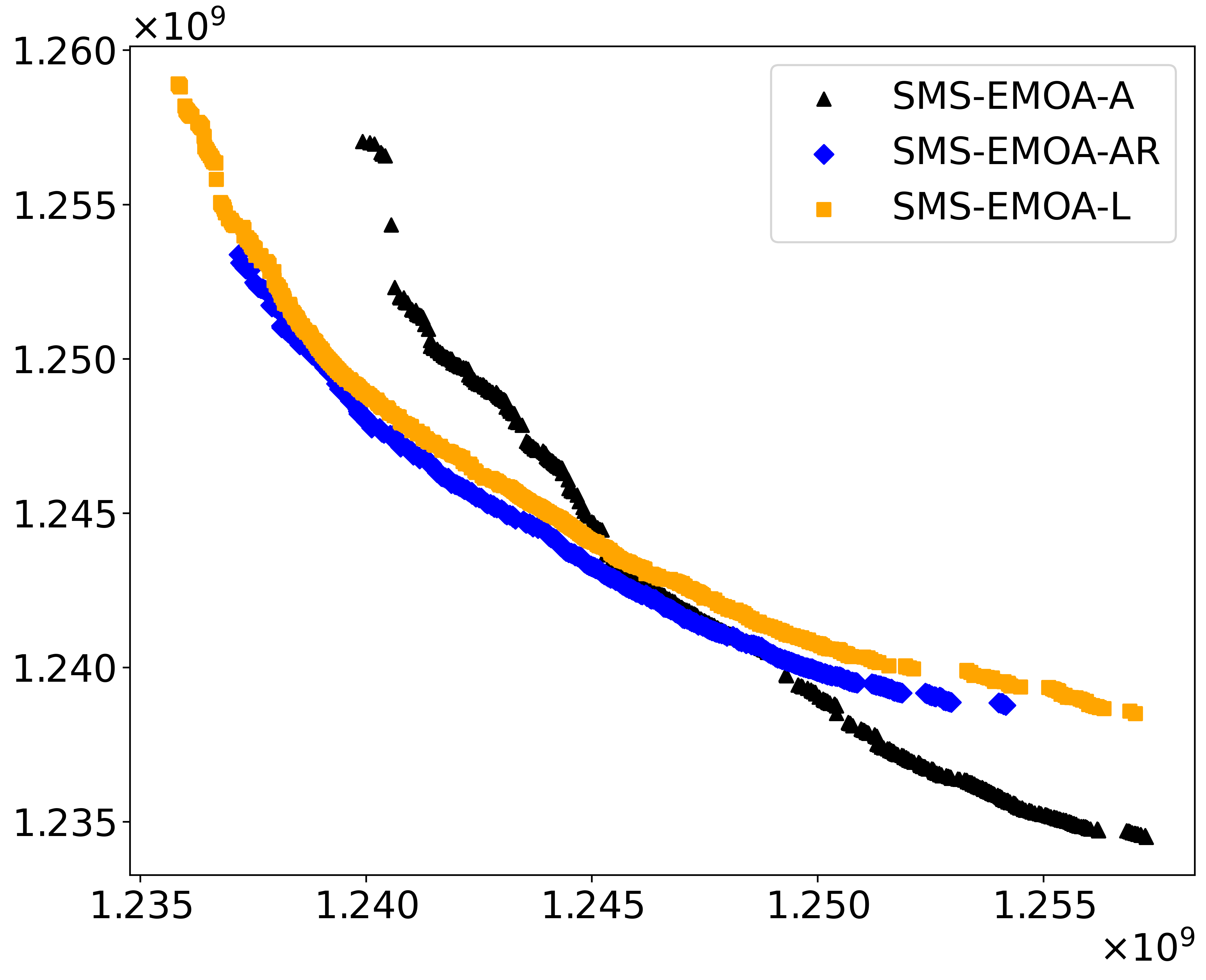} &
            \includegraphics[scale=0.28]{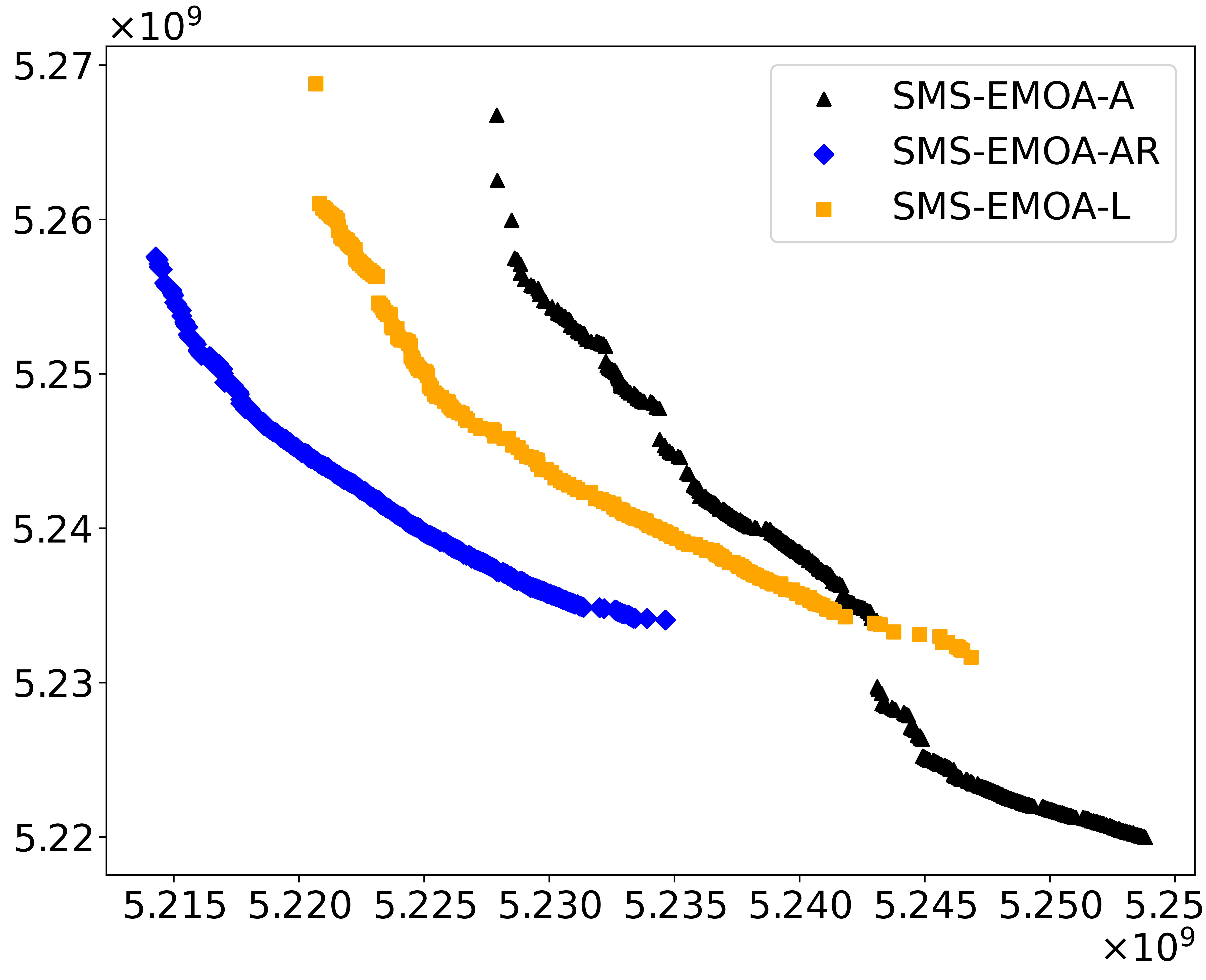} &
            \includegraphics[scale=0.28]{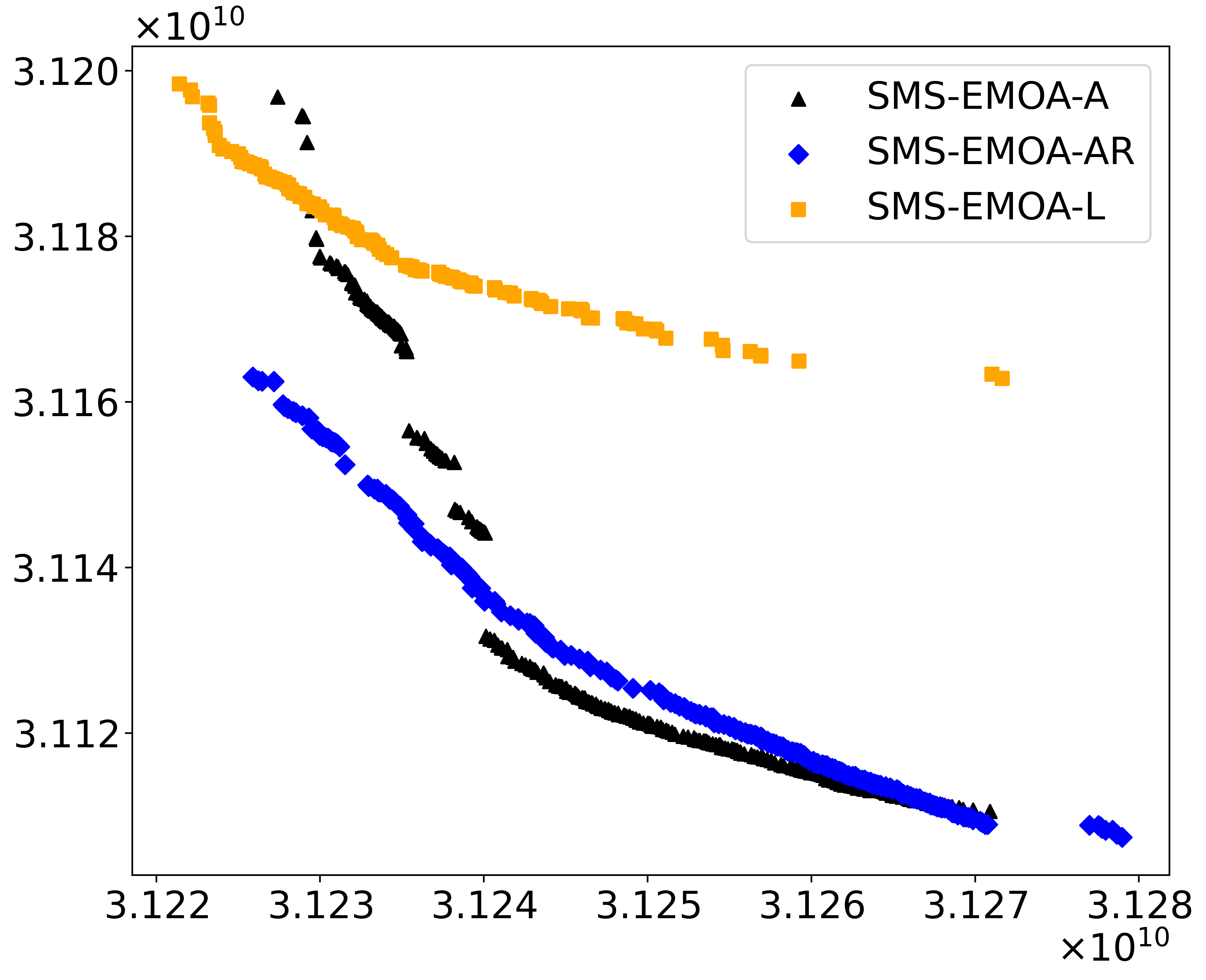} \\
            (j) QAP 100D & (k) QAP 200D & (l) QAP 500D \\
        \end{tabular}
	\end{center}
	\vspace{-10pt}
	\caption{Non-dominated solutions obtained by SMS-EMOA-A (archive only to store non-dominated solutions), SMS-EMOA-AR (archive with reuse), and SMS-EMOA-L (a large population size) on the multi-objective 0/1 Knapsack (KP), NK-Landscape ,(NK), TSP and QAP problems, with 100 (top panel), 200 (middle panel), and 500 (bottom panel) decision variables respectively. The Knapsack and NK-Landscape are maximization problems, and the TSP and QAP are minimization problems.}
	\label{Fig:obj}
\end{figure*}

\section{Experiments}

\begin{table}[t!] 
  \centering
  \begin{tabular}{@{}cccc@{}}
    \toprule
    size $n$ & SMS-EMOA-L & SMS-EMOA-A & SMS-EMOA-AR \\
    \midrule
    15 & 3032.58 & 59244.63 & 2991.15 \\
    20 & 9984.90 & 141679.07 & 6075.74 \\
    25 & 20410.68 & 296777.68 & 9891.29 \\
    30 & 34534.04 & 517542.01 & 15485.38 \\
    \bottomrule
  \end{tabular}
  \caption{Average number of generations over 1,000 independent runs for SMS-EMOA-L (a large population size), SMS-EMOA-A (archive only to store non-dominated solutions), and SMS-EMOA-AR (archive with reuse) solving \aojzj\ with $k=3$ and $a=2$.}\label{tab:aojzj}
\end{table}

\begin{table}[t!] 
  \centering
  \begin{tabular}{@{}cccc@{}}
    \toprule
    size $n$ & SMS-EMOA-L & SMS-EMOA-A & SMS-EMOA-AR \\
    \midrule
    15 & 8232.80 & 12611.38 & 2929.18 \\
    20 & 28375.07 & - & 5541.68 \\
    25 & 56074.58 & - & 8888.95 \\
    30 & 102073.98 & - & 13532.34 \\
    \bottomrule
  \end{tabular}
  \caption{Average number of generations over 1,000 independent runs for SMS-EMOA-L (a large population size), SMS-EMOA-A (archive only to store non-dominated solutions), and SMS-EMOA-AR (archive with reuse) solving \ojzj\ with $k=2$.}\label{tab:ojzj}
\end{table}

In this section, we conduct experiments on two artificial problems \ojzj\ and \aojzj\ analyzed in this paper, and four well-known practical optimization problems: the multi-objective 0/1 knapsack (KP)~\cite{teghem_multi-objective_1994}, travelling salesman problem (TSP)~\cite{ribeiro2002study}, quadratic assignment problem (QAP)~\cite{knowles2003instance} and NK-landscapes (NK)~\cite{Aguirre2004}, following the settings in~\cite{li_empirical_2024}. 
Table~\ref{tab:aojzj} presents the results of SMS-EMOA solving \aojzj\ with problem size $n \in \{ 15, 20, 25, 30\}$, $k = 3$, and $a = 2$, and Table~\ref{tab:ojzj} presents the results of SMS-EMOA soving \ojzj\ with problem size $n\in \{ 15,20,25,30 \}$ and $k=2$.
We present the average number of generations over $1000$ runs for three variants: 1) SMS-EMOA with a large population size $\mu = n-2k+5$, 2) SMS-EMOA with an archive storing all non-dominated solutions ($\mu = 5$), and 3) SMS-EMOA with archive reuse ($\mu = 5$).
For SMS-EMOA-A, since its expected running time on \ojzj\ is at least $n^{\Omega(n)}$, we set the maximum number of iterations to $1,000,000$. For problem sizes $n = \{20, 25, 30\}$, a significant proportion of runs exceed this limit, thus these cases are marked with ``-''.
Results show that reusing the archive significantly reduces running time compared to using an archive only to store non-dominated solutions, and using a large population size.

As for the four practical problems, each problem is instantiated at three sizes: 100, 200, and 500 variables.
We conduct experiments using three SMS-EMOA variants: 1)~With large population size ($\mu=1000$), 2)~Archive without reuse ($\mu=200$), and 3)~Archive with reuse ($\mu=200$). These variants follow common parameter settings: 1) the parent selection uses uniform random selection; 2) the crossover and mutation operators are problem-specific.
For pseudo-Boolean problems (KP, NK), we use the uniform crossover~\cite{eiben2015introduction} with a rate of 1.0 and the bit-wise mutation~\cite{eiben2015introduction} with a rate of $1/n$, where $n$ is the number of variables.
For TSP, we use the order-based crossover and the 2-opt mutation~\cite{eiben2015introduction}.
For QAP, we use the cycle-based crossover and the 2-swap mutation~\cite{eiben2015introduction}. 
For both TSP and QAP, the crossover rate is 1.0 and the mutation rate is 0.05, following the practice in~\cite{li_empirical_2024}.
The budget for the algorithms is 10,000,000 fitness evaluations and the number of independent runs is set to 30. Table~\ref{tab:hv_results} presents the result of a widely used indicator, Hypervolume (HV)~\cite{zitzler1999multiobjective} of the three SMS-EMOA variants on the four problems.
To compute the HV reference point, we sample 100,000 random solutions from the decision space and use the worst objective values among the non-dominated ones.
Table~\ref{tab:hv_results} shows that reusing the archive can lead to a better mean HV across all problem instances.

Figure~\ref{Fig:obj} presents the non-dominated solutions obtained by the SMS-EMOA without archive reuse (black triangles), SMS-EMOA with archive reuse (blue diamonds), and SMS-EMOA with a large population size (yellow squares) on the multi-objective 0/1 Knapsack (KP), NK-Landscape (NK), TSP and QAP problems, with 100 (top panel), 200 (middle panel), and 500 (bottom panel) decision variables respectively.
As can be seen, SMS-EMOA-RA (reusing the archive) consistently outperforms or matches SMS-EMOA (not reusing the archive) across all problem instances. The advantage of archive reuse becomes more pronounced as the problem size increases, particularly on the TSP.

\begin{table}[t!]
  \centering
  \scriptsize
  ``$\dagger$'' indicates that SMS-EMOA$_{A}$ or SMS-EMOA$_{L}$ is significantly worse than SMS-EMOA$_{AR}$ according to the Wilcoxon rank-sum test ($\alpha = 0.05$).
  \small
  \setlength{\tabcolsep}{1mm}
  \begin{tabular}{@{}lccc@{}}
  \toprule
  Problem & SMS-EMOA$_{L}$ & SMS-EMOA$_{A}$ & SMS-EMOA$_{AR}$ \\
  \midrule
  KP$_{100}$     & 2.19e6 (3.2e4)$\dagger$ & 2.24e6 (9.3e3)$\dagger$ & \textbf{2.25e6} (7.7e3) \\
  KP$_{200}$     & 6.44e6 (8.4e4)$\dagger$ & 6.86e6 (2.9e4)$\dagger$ & \textbf{6.88e6} (3.7e4) \\
  KP$_{500}$     & 4.20e7 (4.9e5)$\dagger$ & 4.52e7 (3.1e5)$\dagger$ & \textbf{4.55e7} (3.3e5) \\
  NK$_{100}$  & 6.15e-2 (3.7e-3)$\dagger$ & 7.06e-2 (3.0e-3)$\dagger$ & \textbf{7.23e-2} (3.6e-3) \\
  NK$_{200}$  & 6.04e-2 (3.3e-3)$\dagger$ & 7.24e-2 (2.8e-3) & \textbf{7.36e-2} (2.6e-3) \\
  NK$_{500}$  & 4.45e-2 (2.3e-3)$\dagger$ & 5.91e-2 (1.7e-3) & \textbf{5.92e-2} (1.7e-3) \\
  QAP$_{100}$    & 2.80e15 (9.8e13)$\dagger$ & 3.03e15 (1.4e14)$\dagger$ & \textbf{3.16e15} (1.4e14) \\
  QAP$_{200}$    & 2.28e16 (6.3e14)$\dagger$ & 2.66e16 (8.3e14)$\dagger$ & \textbf{2.76e16} (8.7e14) \\
  QAP$_{500}$    & 2.29e17 (7.3e15)$\dagger$ & 3.16e17 (8.2e15)$\dagger$ & \textbf{3.30e17} (6.4e15) \\
  TSP$_{100}$    & 2.18e3 (4.0e1)$\dagger$ & 2.67e3 (1.7e1) & \textbf{2.68e3} (2.0e1) \\
  TSP$_{200}$    & 6.15e3 (1.8e2)$\dagger$ & 9.03e3 (9.2e1) & \textbf{9.05e3} (7.2e1) \\
  TSP$_{500}$    & 2.35e4 (1.1e3)$\dagger$ & 4.87e4 (8.4e2)$\dagger$ & \textbf{4.92e4} (6.7e2) \\
  \bottomrule
  \end{tabular}
\caption{Hypervolume~\protect\cite{zitzler1999multiobjective} results (mean and standard deviation) of SMS-EMOA$_{L}$ (a large population size), SMS-EMOA$_{A}$ (archive without reuse), and SMS-EMOA$_{AR}$ (archive with reuse) on practical problems. The number in each problem name denotes the number of variables. The best mean per row is highlighted in bold.}
  \label{tab:hv_results}
  \vspace{-0.4cm}
\end{table}

\section{Conclusion}

This paper analytically shows that apart from storing non-dominated solutions, the archive can be reused to benefit the search of MOEAs. We prove that for SMS-EMOA solving \ojzj\ and one of its variants, reusing the archive is significantly superior in terms of expected running time compared to merely using an archive to store non-dominated solutions.
The analysis shows that reusing the archive brings benefits in two ways: 1) it enables revisiting regions missed in the early stages of the search, thus helping to find a better-covered Pareto front; and 2) it alleviates the issue of small populations losing potentially promising non-dominated solutions due to inconsistencies of solution distribution between the decision and objective spaces, thereby enhancing exploration ability. 
We also prove that reusing archive solutions can be better than using a large population size directly.
These theoretical findings are empirically validated on both artificial problems 
and practical problems.

\bibliography{aaai2026}

@book{doerr2020theory,
  title={Theory of Evolutionary Computation: {R}ecent Developments in Discrete Optimization},
  author={Doerr, B. and Neumann, F.},
  year={2020},
  publisher={Springer}
}

@inproceedings{zheng2024sms,
	title={Runtime analysis of the {SMS-EMOA} for many-objective optimization},
	author={Zheng, W. and Doerr, B.},
	booktitle={AAAI},
  pages={20874--20882},
	address={Vancouver, Canada},
	year={2024}
}

@inproceedings{bian2024archive,
  title     = {An archive can bring provable speed-ups in multi-objective evolutionary algorithms},
  author    = {Bian, C. and Ren, S. and Li, M. and Qian, C.},
  booktitle = {IJCAI},
  pages     = {6905--6913},
address={Jeju Island, South Korea},
  year      = {2024},
}

@article{bian23stochastic,
  title={Stochastic population update can provably be helpful in multi-objective evolutionary algorithms},
  author={Bian, C. and Zhou, Y. and Li, M. and Qian, C.},
  journal={Artificial Intelligence},
volume = {341},
  pages={104308},
  year={2025},
}

@inproceedings{lu2024imoea,
	title={Towards running time analysis of interactive multi-objective evolutionary algorithms},
	author={Lu, T. and Bian, C. and Qian, C.},
	booktitle={AAAI},
	pages={20777-20785},
	address={Vancouver, Canada},
	year={2024}
}

@article{deb2014nsgaiii,
	title={An evolutionary many-objective optimization algorithm using reference-point-based nondominated sorting approach, part {I}: {S}olving problems with box constraints},
	author={Deb, K. and Jain, H.},
	journal={IEEE Transactions on Evolutionary Computation},
	volume={18},
	number={4},
	pages={577--601},
	year={2014}
}

@inproceedings{dang2023analysing,
	author = {Dang, D.-C. and Opris, A. and Salehi, B. and Sudholt, D.},
	title = {Analysing the robustness of {NSGA-II} under noise},
	year = {2023},
	booktitle={GECCO},
	pages = {642--651},
	address = {Lisbon, Portugal},
}

@inproceedings{wietheger23nsgaiii,
	title={A mathematical runtime analysis of the non-dominated sorting genetic algorithm {III (NSGA-III)}},
	author={Wietheger, S. and Doerr, B.},
	booktitle={IJCAI},
	year={2023},
	pages={5657--5665},
	address = {Macao, SAR, China}
}

@article{zheng2023manyobj,
	title={Runtime analysis for the {NSGA-II}: {P}roving, quantifying, and explaining the inefficiency for many objectives},
	journal={IEEE Transactions on Evolutionary Computation},
	author={Zheng, W. and Doerr, B.},
	year={2024},
  volume={28},
  number={5},
  pages={1442-1454},
}

@inproceedings{Opris2024nsgaiii,
  title={Runtime analyses of {NSGA-III} on many-objective problems},
  author={Opris, A. and Dang, D.-C.  and Sudholt, D.},
  booktitle={GECCO},
  pages={1596--1604},
address={Melbourne, Australia},
  year={2024}
}

@incollection{deb2005scalable,
  title={Scalable test problems for evolutionary multiobjective optimization},
  author={Deb, K. and Thiele, L. and Laumanns, M. and Zitzler, E.},
  booktitle={Evolutionary multiobjective optimization: theoretical advances and applications},
  pages={105--145},
  year={2005},
publisher={Springer}
}

@inproceedings{li2019empirical,
  title={An empirical investigation of the optimality and monotonicity properties of multiobjective archiving methods},
  author={Li, M. and Yao, X.},
  booktitle={International conference on evolutionary multi-criterion optimization},
  pages={15--26},
  year={2019},
}

@article{figueira2017easy,
  title={Easy to say they are hard, but hard to see they are easy—towards a categorization of tractable multiobjective combinatorial optimization problems},
  author={Figueira, J. and Fonseca, C. M. and Halffmann, P. and Klamroth, K. and Paquete, L. and Ruzika, S. and Schulze, B. and Stiglmayr, M. and Willems, D.},
  journal={Journal of Multi-Criteria Decision Analysis},
  volume={24},
  number={1-2},
  pages={82--98},
  year={2017},
}

@book{deb2001book,
	title={Multi-objective Optimization using Evolutionary Algorithms},
	author={K. Deb},
	year={2001},
	publisher={Wiley}
}

@inproceedings{bian2022better,
	title={Better running time of the non-dominated sorting genetic algorithm {II (NSGA-II)} by using stochastic tournament selection},
	author={Bian, C. and Qian, C.},
	booktitle={PPSN},
	pages={428--441},
	address={Dortmund, Germany},
	year={2022}
}

@inproceedings{qian2015subset,
	title={Subset selection by Pareto optimization},
	author={Qian, C. and Yu, Y. and Zhou, Z.-H.},
	booktitle={NIPS},
	pages={1765--1773},
	address={Montreal,Canada},
	year={2015}
}

@article{qian2019sub,
	title={Maximizing submodular or monotone approximately submodular functions by multi-objective evolutionary algorithms},
	author={Qian, C. and Yu, Y. and Tang, K. and Yao, X. and Zhou, Z.-H.},
	journal={Artificial Intelligence},
	volume={275},
	pages={279--294},
	year={2019}
}

@article{Friedrich2015sub,
	title={Maximizing submodular functions under matroid constraints by evolutionary algorithms},
	author={Friedrich, T. and Neumann, F.},
	journal={Evolutionary Computation},
	volume={23},
	number={4},
	pages={543--558},
	year={2015}
}

@inproceedings{qian2017subset,
	title={Subset selection under noise},
	author={Qian, C. and Shi, J.-C. and Yu, Y. and Tang, K. and Zhou, Z.-H.},
	booktitle={NIPS},
	pages={3562--3572},
	address={Long Beach, CA},
	year={2017}
}

@InProceedings{Ren2024spea2,
author={Ren, S. and Bian, C. and Li, M. and Qian, C.},
title={A first running time analysis of the strength {P}areto evolutionary algorithm 2 ({SPEA}2)},
booktitle={PPSN},
year={2024},
address={Hagenberg, Austria},
pages={295--312},
}

@article{beume2007sms,
	title={{SMS-EMOA: M}ultiobjective selection based on dominated hypervolume},
	author={Beume, N. and Naujoks, B. and Emmerich, M.},
	journal={European Journal of Operational Research},
	volume={181},
	pages={1653--1669},
	year={2007}
}

@inproceedings{doerr2021ojzj,
	title={Theoretical analyses of multi-objective evolutionary algorithms on multi-modal objectives},
	author={Doerr, B. and Zheng, W.},
	booktitle={AAAI},
	pages={12293--12301},
	address={Virtual},
	year={2021}
}

@inproceedings{cerf2023first,
	title={The first proven performance guarantees for the non-dominated sorting genetic algorithm {II (NSGA-II)} on a combinatorial optimization problem},
	author={Cerf, S. and Doerr, B. and Hebras, B. and Kahane, Y. and Wietheger, S.},
	booktitle={IJCAI},
	pages={5522--5530},
	address={Macao, SAR, China},
	year={2023}
}

@inproceedings{dang2023crossover,
	title={A proof that using crossover can guarantee exponential speed-ups in evolutionary multi-objective optimisation},
	author={Dang, D.-C. and Opris, A. and Salehi, B. and Sudholt, D.},
	booktitle={AAAI},
        pages={12390--12398},
	address={Washington, DC},
	year={2023}
}

@article{opris2025first,
  title={A first runtime analysis of the {PAES}-25: An enhanced variant of the Pareto archived evolution strategy},
  author={Opris, A.},
  journal={CORR abs/2507.03666},
  year={2025}
}

@inproceedings{doerr2023crossover,
	title={Runtime analysis for the {NSGA-II}: {P}rovable speed-ups from crossover},
	author={Doerr, B. and Qu, Z.},
	booktitle={AAAI},
        pages={12399--12407},
	address={Washington, DC},
	year={2023}
}

@inproceedings{doerr2023lower,
	title={From understanding the population dynamics of the {NSGA-II} to the first proven lower bounds},
	author={Doerr, B. and Qu, Z.},
	booktitle={AAAI},
	pages={12408--12416},
	address={Washington, DC},
	year={2023}
}

@article{zitzler1999multiobjective,
	title={Multiobjective evolutionary algorithms: {A} comparative case study and the strength {Pareto} approach},
	author={Zitzler, E. and Thiele, L.},
	journal={IEEE Transactions on Evolutionary Computation},
	volume={3},
	number={4},
	pages={257--271},
	year={1999}
}

@book{eiben2015introduction,
	title={Introduction to Evolutionary Computing},
	author={Eiben, A. E. and Smith, J. E.},
	year={2015},
        publisher="Springer Berlin Heidelberg",
        address="Heidelberg, Germany",
}

@book{qian19el,
	author={Z.-H. Zhou and Y. Yu and C. Qian},
	year={2019},
	title={Evolutionary Learning: {A}dvances in Theories and Algorithms},
	publisher={Springer},
}

@article{deb-tec02-nsgaii,
	author = {K. Deb and A. Pratap and S. Agarwal and T. Meyarivan},
	title = {A fast and elitist multiobjective genetic algorithm: {NSGA-II}},
	journal={IEEE Transactions on Evolutionary Computation},
	volume = {6},
	number={2},
	pages = {182--197},
	year = {2002}
}

@article{zhang2007moea,
	title={{MOEA/D}: {A} multiobjective evolutionary algorithm based on decomposition},
	author={Q. Zhang and H. Li},
	journal={IEEE Transactions on Evolutionary Computation},
	volume={11},
	number={6},
	pages={712--731},
	year={2007}
}

@inproceedings{zheng2021first,
	title={A first mathematical runtime analysis of the non-dominated sorting genetic algorithm {II (NSGA-II)}},
	author={Zheng, W. and Liu, Y. and Doerr, B.},
	booktitle={AAAI},
	pages={10408--10416},
	address={Virtual},
	year={2022}
}

@inproceedings{ren2024multimodel,
	title={Maintaining diversity provably helps in evolutionary multimodal optimization},
	author={Ren, S. and Qiu,Z. and Bian, C. and Li, M. and Qian, C.},
	booktitle={IJCAI},
	year={2024},
    address = {Jeju Island, South Korea},
    pages = {7012-7020},
}

@inproceedings{doerr2025runtime,
  title={Runtime analysis for multi-objective evolutionary algorithms in unbounded integer spaces},
  author={Doerr, B. and Krejca, M. S. and Rudolph, G.},
  booktitle={AAAI},
  number={25},
  pages={26955--26963},
  year={2025},
address={Philadelphia, Pennsylvania}
}

@inproceedings{dang2024level,
  title={Level-based theorems for runtime analysis of multi-objective evolutionary algorithms},
  author={Dang, D.-C. and Opris, A. and Sudholt, D.},
  booktitle={PPSN},
  pages={246--263},
  year={2024},
  address={Hagenberg, Austria},
}

@inproceedings{wietheger2024near,
  title={Near-tight runtime guarantees for many-objective evolutionary algorithms},
  author={Wietheger, S. and Doerr, B.},
  booktitle={PPSN},
  pages={153--168},
  year={2024},
  address={Hagenberg, Austria},
}

@inproceedings{doerr2025speeding,
  title={Speeding up the NSGA-II with a simple tie-breaking rule},
  author={Doerr, B. and Ivan, T. and Krejca, M. S.},
  booktitle={AAAI},
  number={25},
  pages={26964--26972},
  year={2025},
address={Philadelphia, Pennsylvania}
}

@inproceedings{ren2025archive,
	title={A theoretical perspective on why stochastic population update needs an archive in evolutionary multi-objective optimization},
	author={Ren, S. and Liang, Z. and Li, M. and Qian, C.},
	booktitle={IJCAI},
	year={2025},
    address = {Montreal, Canada},
}

@inproceedings{bian2018tools,
	title={A general approach to running time analysis of
	multi-objective evolutionary algorithms},
	author={C. Bian and C. Qian and K. Tang},
	booktitle={IJCAI},
	pages={1405--1411},
	address={Stockholm, Sweden},
	year={2018}
}

@inproceedings{huang2021runtime,
	title={A runtime analysis of typical decomposition approaches in {MOEA/D} framework for many-objective optimization problems},
	author={Huang, Z. and Zhou, Y. and Luo, C. and Lin, Q.},
	booktitle={IJCAI},
	pages={1682--1688},
	address={Virtual },
	year={2021}
}

@article{Qian13,
	author = {C. Qian and Y. Yu and Z.-H. Zhou},
	title = "An analysis on recombination in multi-objective evolutionary optimization",
	journal = "Artificial Intelligence",
	volume = "204",
	pages = "99--119",
	year = "2013"
}

@Inproceedings{qian-ppsn16-hyper,
	author={C. Qian and K. Tang and Z.-H. Zhou},
	title={Selection hyper-heuristics can provably be helpful
	in evolutionary multi-objective optimization},
	booktitle={PPSN},
	pages={835-846},
	address={Edinburgh, Scotland},
	year={2016}
}

@article{LaumannsTEC04,
	author={M. Laumanns and L. Thiele and E. Zitzler},
	title={Running time analysis of multiobjective evolutionary algorithms on pseudo-{B}oolean functions},
	journal={IEEE Transactions on Evolutionary Computation},
	volume={8},
	number={2},
	pages={170--182},
	year={2004}
}

@inproceedings{doerr2013lower,
	title={Lower bounds for the runtime of a global multi-objective evolutionary algorithm},
	author={Doerr, B. and Kodric, B. and Voigt, M.},
	booktitle={CEC},
	pages={432--439},
	year={2013},
	address={Cancun, Mexico}
}

@article{laumanns-nc04-knapsack,
	author={M. Laumanns and L. Thiele and E. Zitzler},
	title={Running time analysis of evolutionary algorithms on a
	simplified multiobjective knapsack problem},
	journal={Natural Computing },
	volume={3},
	pages={37--51},
	year={2004}
}

@Inproceedings{Neumann10,
	author={F. Neumann and M. Theile},
	title="How crossover speeds up evolutionary algorithms for the multi-criteria all-pairs-shortest-path problem",
	booktitle="PPSN",
	pages="667--676",
	address="Krakow, Poland",
	year="2010"
}

@article{Giel10,
	title={On the effect of populations in evolutionary multi-objective optimisation},
	author={O. Giel and P. K. Lehre},
	journal={Evolutionary Computation},
	volume={18},
	number={3},
	pages={335--356},
	year={2010}
}

@article{Neumann07,
	author = {F. Neumann},
	title = "Expected runtimes of a simple evolutionary algorithm for the multi-objective minimum spanning tree problem",
	journal = "European Journal of Operational Research",
	volume = "181",
	number = "3",
	pages = "1620--1629",
	year = "2007"
}

@inproceedings{ishibuchi2020new,
  title={A new framework of evolutionary multi-objective algorithms with an unbounded external archive},
  author={Ishibuchi, H. and Pang, L. M. and Shang, K.},
  booktitle={ECAI 2020},
  pages={283--290},
  year={2020},
address={Santiago, Spain}
}

@inproceedings{brockhoff2019benchmarking,
  title={Benchmarking algorithms from the platypus framework on the biobjective bbob-biobj testbed},
  author={Brockhoff, D. and Tu{\v{s}}ar, T.},
  booktitle={GECCO},
  pages={1905--1911},
address={Prague, Czech Republic},
  year={2019}
}

@inproceedings{krause2016unbounded,
  title={Unbounded population {MO-CMA-ES} for the bi-objective BBOB test suite},
  author={Krause, O. and Glasmachers, T. and Hansen, N. and Igel, C.},
  booktitle={GECCO},
  pages={1177--1184},
address={Denver, CO},
  year={2016},
}

@article{fieldsend2003using,
  title={Using unconstrained elite archives for multiobjective optimization},
  author={Fieldsend, J. E. and Everson, R. M. and Singh, S.},
  journal={IEEE Transactions on Evolutionary Computation},
  volume={7},
  number={3},
  pages={305--323},
  year={2003},
}

@ARTICLE{doerr2023nsgaojzj,
  author={Doerr, B. and Qu, Z.},
  journal={IEEE Transactions on Evolutionary Computation}, 
  title={A first runtime analysis of the {NSGA-II} on a multimodal problem}, 
  year={2023},
  volume={27},
  number={5},
  pages={1288-1297},
  }

@article{knowles2003properties,
  title={Properties of an adaptive archiving algorithm for storing nondominated vectors},
  author={Knowles, J. and Corne, D.},
  journal={IEEE Transactions on Evolutionary Computation},
  volume={7},
  number={2},
  pages={100--116},
  year={2003},
}

@article{li2023multi,
  title={Multi-objective archiving},
  author={Li, M. and L{\'o}pez-Ib{\'a}{\~n}ez, M. and Yao, X.},
  journal={IEEE Transactions on Evolutionary Computation},
 year={2024},
  volume={28},
  number={3},
  pages={696-717},
}

@InProceedings{Aguirre2004,
  Title                    = {Effects of elitism and population climbing on multiobjective {MNK}-landscapes},
  Author                   = {Aguirre, H. E. and Tanaka, K.},
  Booktitle                = {CEC},
  address={Portland, OR},
	year={2004},
  Pages                    = {449--456},
  Volume                   = {1}
}

@article{ribeiro2002study,
  title={A study of global convexity for a multiple objective travelling salesman problem},
  author={Ribeiro, C. C. and Hansen, P. and Borges, P. C. and Hansen, M. P.},
  journal={Essays and Surveys in Metaheuristics},
  pages={129--150},
  year={2002},
}

@article{teghem_multi-objective_1994,
	title = {Multi-objective combinatorial optimization problems: {A} survey},
	language = {en},
	journal = {Journal of Multi-Criteria Decision Analysis},
	author = {Teghem, Jacques},
	month = jan,
	year = {1994},
}

@inproceedings{knowles2003instance,
  title={Instance generators and test suites for the multiobjective quadratic assignment problem},
  author={Knowles, Joshua and Corne, David},
  booktitle={Evolutionary Multi-Criterion Optimization},
  pages={295--310},
  year={2003},
}

@article{robbins1955remark,
  title={A remark on Stirling's formula},
  author={Robbins, H.},
  journal={The American mathematical monthly},
  volume={62},
  number={1},
  pages={26--29},
  year={1955},
}

@inproceedings{li_empirical_2024,
	title = {Empirical {comparison} between {MOEAs} and {local} {search} on {multi}-{objective} {combinatorial} {optimisation} {problems}},
	booktitle = {GECCO},
	publisher = {ACM},
	author = {Li, Miqing and Han, Xiaofeng and Chu, Xiaochen and Liang, Zimin},
	year = {2024},
	pages = {547--556},
}

\clearpage

\appendix
\section{Proof of Theorem~\ref{thm:sms-arc-omm} and Theorem~\ref{thm:sms-arc-lotz}}

\subsection{\omm\ and \lotz}

The \omm\ problem presented in Definition~\ref{def:OMM} aims to simultaneously maximize the number of 0-bits and the number of 1-bits of a binary bit string.  
The Pareto front is $\{(b, n-b)\mid b\in [0..n]\}$, whose size is $n+1$, and 
the Pareto optimal solution corresponding to $(b, n-b)$, $b\in [0..n]$, is any solution with $(n-b)$ 1-bits. We can see that any solution $\bmx\in\{0,1\}^n$ is Pareto optimal for this problem.
\begin{definition}[\omm~\cite{Giel10}]\label{def:OMM}
	The OneMinMax problem of size $n$ is to find $n$ bits binary strings which maximize
        $
		{\bm f}(\bmx)=\left(n-\sum\nolimits^n_{i=1}x_i, \sum\nolimits^{n}_{i=1} x_i\right)
        $,
	where $x_i$ denotes the $i$-th bit of $\bmx \in \{0,1\}^n$.
\end{definition}

The \lotz\  problem presented in Definition~\ref{def:LOTZ} aims to simultaneously maximize the number of leading 1-bits and the number of trailing 0-bits of a binary bit string.  
The Pareto front is $\{(a, n-a)\mid a\in [0..n]\}$, whose size is $n+1$, and the Pareto optimal solution corresponding to $(a,n-a)$,  $a\in [0..n]$, is $1^a0^{n-a}$, i.e., the solution with $a$ leading 1-bits and $n-a$ trailing 0-bits.
\begin{definition}[\lotz~\cite{LaumannsTEC04}]\label{def:LOTZ}
	The \lotz\ problem of size $n$ is to find $n$ bits binary strings which maximize
        $
        	{\bm{f}}(\bmx)= (\sum\nolimits^n_{i=1} \prod\nolimits^{i}_{j=1}x_j, \sum\nolimits^{n}_{i=1} \prod\nolimits^{n}_{j=i}(1-x_j))
         $,
	where $x_j$ denotes the $j$-th bit of $\bmx \in \{0,1\}^n$.
\end{definition}

\begin{algorithm}[h!]
	\caption{SMS-EMOA with archive reues}
	\label{alg:sms_crossover}
	\textbf{Input}: objective function $f_1,f_2\cdots,f_m$, population size $\mu$, probability $p_c$ of using crossover \\
	\textbf{Output}: $\mu$ solutions from $\{0,1\}^n$
	\begin{algorithmic}[1] 
		\STATE $P\leftarrow \mu$ solutions uniformly and randomly selected from $\{0, 1\}^{\!n}$ with replacement, $A\leftarrow \emptyset$;
		\WHILE{criterion is not met}
		\STATE $\bmx \leftarrow$ \textsc{Selection with Archive Reuse} $(P,A)$;
		\STATE sample $u$ from the uniform distribution over $[0, 1]$;
		\IF{$u<p_c$}
		\STATE $\bmy \leftarrow$ \textsc{Selection with Archive Reuse} $(P,A)$;
		\STATE apply one-point crossover on $\bmx$ and $\bmy$ to generate  $\bmx'$
		\ELSE 
		\STATE set $\bmx'$ as the copy of $\bmx$
		\ENDIF
		\STATE apply bit-wise mutation on $\bmx'$ to generate $\bmx''$;
        \IF{$\not \exists \bmz \in A$ such that $\bmz \succeq \bmx''$}
        \STATE $A \leftarrow (A \setminus\{\bmz \in A \mid \bmx'' \succ \bmz\}) \cup \{\bmx''\}$
        \ENDIF
		\STATE let $\bmz=\arg\min_{\bmx\in R_v}\Delta_{\bmr}(\bmx,R_v)$;
		\STATE $P\leftarrow (P\cup \{\bmx''\})\setminus \{\bmz\}$
		\ENDWHILE
		\RETURN $A$
	\end{algorithmic}
\end{algorithm}

\begin{algorithm}[t!]
\caption{\textsc{Selection with Archive Reuse}} \label{archive_selection}
	\textbf{Input}: population $P$, archive $A$\\
	\textbf{Output}: one solution from $\{0,1\}^n$
	\begin{algorithmic}[1] 
	\STATE sample $u$ from uniform distribution over $[0,1]$;
      \IF{$u<1/2$}
      \STATE select a solution $\bm{x}$ from $P$ uniformly at random;
      \ELSE
      \STATE select a solution $\bm{x}$ from $A$ uniformly at random;
  \ENDIF
	\end{algorithmic}
\end{algorithm}

Note that the analysis of Theorems~\ref{thm:sms-arc-omm} and~\ref{thm:sms-arc-lotz} is very similar to that in~\cite{bian2024archive}. Although so, we still provide the proof in detail for completeness. The SMS-EMOA algorithm with archive reuse and one-point crossover operator is presented in Algorithm~\ref{alg:sms_crossover}. The archive reuse mechanism is adapted such that each parent solution for recombination is selected with probability $1/2$ from the population and with probability $1/2$ from the archive. Note that one-point crossover selects a random point $i \in \{1, 2, \ldots, n\}$ and exchanges the first $i$ bits of two parent solutions, which actually produces two new solutions, but the algorithm only picks the one that consists of the first part of the first parent solution and the second part of the second parent solution.

\begin{proof}[Proof of Theorem~\ref{thm:sms-arc-omm}]
    We divide the running process into two phases. The first phase starts after initialization and finishes until $1^n$ and $0^n$ are both found; the second phase starts after the second phase and finishes until the whole Pareto front is found.

   For the first phase, we prove that the expected number of generations for finding $1^n$ is $O(\mu n\log n)$, and the same bound holds for finding $0^n$ analogously. Since SMS-EMOA always preserves the two boundary points, i.e., the two solutions with the maximum $f_1$ and $f_2$ values respectively. Thus, the maximum $f_2$ value among the Pareto optimal solutions in $P\cup \{\bm{x}''\}$ will not decrease, where $\bm{x}''$ is the offspring solution generated in each iteration. Now, we consider the increase of the maximum $f_2$ value among the Pareto optimal solutions found. In each generation, SMS-EMOA chooses the population $P$ as the parent population (whose probability is $1/2$), selects the Pareto optimal solution $\bm{x}$ with the maximum $f_2$ value as a parent solution (whose probability is $1/\mu$), and generates a solution with more 1-bits if crossover is not performed (whose probability is $1-p_c$) and only one of the 0-bits in $\bmx$ is flipped by bit-wise mutation (whose probability is $((n-|\bm{x}|_1)/n)\cdot (1-1/n)^{n-1}$). This implies that the probability of generating a solution with more than $|\bm{x}|_1$ 1-bits in one generation is at least
   \begin{equation}
     \frac{1}{2}\cdot\frac{1}{\mu}\cdot (1-p_c)\cdot \frac{n-|\bm{x}|_1}{n}\cdot \Big(1-\frac{1}{n}\Big)^{n-1}\ge (1-p_c)\cdot \frac{n-|\bm{x}|_1}{2e\mu n}.
   \end{equation}
   Thus, the expected number of generations for increasing the maximum $f_2$ value to $n$, i.e., finding a solution with $n$ 1-bits, is at most $(1/ (1-p_c)) \cdot \sum_{i=0}^{n-1} 2e\mu n/(n-i)=O(\mu n \log n)$, where the equality holds by $p_c=\Theta(1)$. That is, the expected number of generations of the first phase is $O(\mu n\log n)$. 

    Then, we consider the second phase, and will show that the SMS-EMOA can find the whole Pareto front in $O(\mu n\log n)$ expected number of generations. Note that after the second phase, $0^n$ and $1^n$ must be maintained in the population $P$. Let $D=\{j \mid \exists \bm{x}\in A, |\bmx|_1 = j\}$, where $A$ denotes the archive, and we suppose $|D|=i$, i.e., $i$ points on the Pareto front have been found in the archive. Note that $i \geq 2$ as $0^n$ and $1^n$ have been found. Assume that in the reproduction procedure, a solution $\bm{x}$ with $|\bm{x}|_1 = j$ ($j \in [0..n]$) is selected as one parent. If the other selected parent solution is $1^n$, then for any $d\in [1..n-1] \setminus D$, there must exist a crossover point $d'$ such that exchanging the first $d'$ bits of $\bm{x}$ and $1^n$ can generate a solution with $d$ 1-bits. If the other selected parent is $0^n$, then for any $d\in [1..n-1] \setminus D$, there must exist a crossover point $d'$ such that exchanging the first $d'$ bits of $\bm{x}$ and $0^n$ can generate a solution with $d$ 1-bits. The newly generated solution can keep unchanged if no bits are flipped in bit-wise mutation. Note that the probability of choosing the population $P$ as the parent population selecting $1^n$ (or $0^n$) as a parent solution is $1/2\mu$. Thus, the probability of generating a new point on the Pareto front is at least
\begin{equation}
\begin{aligned}
  &\frac{1}{2\mu}\cdot p_c \cdot \frac{n+1 - |D|}{n} \cdot \Big(1-\frac{1}{n}\Big)^n \\
  &\ge \frac{p_c(n+1-|D|)}{4e\mu n}.
\end{aligned}
\end{equation}
Then, we can derive that the expected number of generations of the third phase (i.e., for finding the whole Pareto front) is at most $\sum_{i=2}^{n} O(4e\mu n/(n+1-i)) = O(\mu n\log n)$.

Combining the two phases, the total expected number of generations is $O(\mu n \log n+ \mu n\log n) = O( \mu n\log n)$, where $\mu \ge 2$. Thus, the theorem holds.
\end{proof}

\begin{proof}[Proof of Theorem~\ref{thm:sms-arc-lotz}]
     We divide the running process into two phases. The first phase starts after initialization and finishes until $1^n$ and $0^n$ are both found; the second phase starts after the second phase and finishes until the whole Pareto front is found.

    The proof of the first phase is similar to that of Theorem~\ref{thm:sms-arc-omm}. Firstly, we consider the increase of $f_1$ value, i.e., $\max_{\bmx\in P} \text{LO}(\bmx)$. In each generation, SMS-EMOA chooses the population $P$ as the parent solution, selects a solution with maximum $f_1$ value, and flips its $(\text{LO}(\bmx) +1)$-th bit (which must be $0$) without using crossover. with probability at least
    \begin{align}
        \frac{1}{2\mu} \cdot (1-p_c)\cdot \frac{1}{n}\cdot \Big(1-\frac{1}{n}\Big)^{n-1} \ge \frac{1-p_c}{2e\mu n}.
    \end{align}
    Thus, the expected number of generations for the first phase is $\sum_{i=0}^{n-1} 2e\mu n/(1-p_c) = O(\mu n^2)$.

    Next, we consider the second phase. Let $D = \{j\in [0..n]\mid 1^j0^{n-j}\in A\}$, where $A$ denotes the archive, and suppose $|D| = i$, i.e., $i$ points on the Pareto front has been found in the archive. By selecting $1^n$ and $0^n$ from population $P$ as a pair of parent solutions, and exchanging their first $k$ bits ($k\in [0..n]\backslash D$) by one-point crossover, the solution $1^k0^{n-k}$ can be generated. $1^k0^{n-k}$ can keep unchanged by flipping no bits in bit-wise mutation. The probability of choosing population $P$ as parent population and selecting $1^n$ and $0^n$ as a pair of parents is $1/(2\mu)^2$. Thus, a new point on the Pareto front can be generated with probability at least
    \begin{align}
        \frac{1}{(2\mu)^2} \cdot p_c \cdot \frac{n+1-i}{n}\cdot \Big(1-\frac{1}{n}\Big)^{n-1} \ge \frac{n+1-i}{8e\mu^2 n}.
    \end{align}
    Thus, the expected number of generations of finding the whole Pareto front is $\sum_{i=2}^n 8e\mu^2 n/(n+1-i) = (\mu^2 n\log n)$.

    Combining the two phases, the total expected number of generations is $O(\mu n^2+ \mu^2 n\log n)$, where $\mu \ge 2$. Thus, the theorem holds.
\end{proof}

\end{document}